\documentclass[letterpaper, 10 pt, journal, twoside]{IEEEtran}

\IEEEoverridecommandlockouts                              %

\usepackage[utf8]{inputenc}
\usepackage[english]{babel}
\usepackage{graphicx}
\usepackage{amsmath,amssymb}
\usepackage{amsfonts}
\usepackage{upgreek}
\usepackage{color}
\usepackage{marginnote}
\usepackage[ruled,vlined]{algorithm2e}
\usepackage{multirow}
\usepackage{paralist}
\usepackage{wrapfig}
\usepackage{xcolor}
\usepackage{hyperref}

\let\labelindent\relax
\usepackage{amsthm}
\usepackage{environ}
\usepackage[normalem]{ulem}
\usepackage{enumitem}
\usepackage{subcaption}
\usepackage{url}
\usepackage{hyperref}

\newtheorem{theorem}{Theorem}
\newtheorem{lemma}{Lemma}
\newtheorem{definition}{Definition}

\newtheorem*{remark}{Remark}

\usepackage{fancyvrb}
\usepackage{listings}
\newcommand{\norm}[1]{\left\lVert#1\right\rVert}

\usepackage{scalerel,stackengine}
\stackMath

\newcommand{\poly}{\texttt{Poly}}

\newcommand{\polydim}[1]{\texttt{Row}(#1)}

\newcommand{\revise}[1]{#1}

\newcommand{\nnnum}[1]{\relax\ifmmode 
	{\mathbb #1}_{\geq 0} \else ${\mathbb #1}_{\geq 0}$
	\fi}
\newcommand{\npnum}[1]{\relax\ifmmode 
	{\mathbb #1}_{\leq 0} \else ${\mathbb #1}_{\leq 0}$
	\fi}
\newcommand{\pnum}[1]{\relax\ifmmode 
	{\mathbb #1}_{> 0} \else ${\mathbb #1}_{> 0}$
	\fi}
\newcommand{\nnum}[1]{\relax\ifmmode 
	{\mathbb #1}_{< 0} \else ${\mathbb #1}_{< 0}$
	\fi}
\newcommand{\plnum}[1]{\relax\ifmmode 
	{\mathbb #1}_{+} \else ${\mathbb #1}_{+}$
	\fi}
\newcommand{\nenum}[1]{\relax\ifmmode 
	{\mathbb #1}_{-} \else ${\mathbb #1}_{-}$
	\fi}

\newcommand{\reals}{{\mathbb{R}}}                    %
\newcommand{\nnreals}{{\nnnum R}}                    %

\renewcommand{\emptyset}{\varnothing}

\newcommand\reallywidehat[1]{%
\savestack{\tmpbox}{\stretchto{%
  \scaleto{%
    \scalerel*[\widthof{\ensuremath{#1}}]{\kern-.6pt\bigwedge\kern-.6pt}%
    {\rule[-\textheight/2]{1ex}{\textheight}}%
  }{\textheight}%
}{0.5ex}}%
\stackon[1pt]{#1}{\tmpbox}%
}

\newcommand{\sat}{\vDash}
\newcommand{\until}{\mathcal{U}}
\newcommand{\release}{\mathcal{R}}
\newcommand{\eventually}{\Diamond}
\newcommand{\always}{\square}

\newcommand{\RN}[1]{%
  \textup{\uppercase\expandafter{\romannumeral#1}}%
}

\NewEnviron{smallalign}{%
\scalebox{0.8}{
\begin{align*}
\BODY
\end{align*}}}

\newcommand{\extended}{}

\title{Multi-agent Motion Planning from Signal Temporal Logic Specifications}

\begin{document}

\author{Dawei Sun$^{1}$, Jingkai Chen$^{2}$, Sayan Mitra$^{1}$, and Chuchu Fan$^{2}$%
\thanks{Manuscript received: September 9, 2021; Revised December 12, 2021; Accepted January 6, 2022.}%
\thanks{This paper was recommended for publication by Editor M. Ani Hsieh upon evaluation of the Associate Editor and Reviewers' comments. Sun and Mitra were supported by research grants from the National Security Agency’s Science of Security (SoS) program and National Science Foundation’s Formal Methods in the Field (FMITF) program. Fan was supported by MIT-IBM Watson AI Lab.} %
\thanks{$^{1}$Dawei Sun and Sayan Mitra are with University of Illinois at Urbana-Champaign, Champaign, IL, 61820 USA.
        {\tt\footnotesize \{daweis2, mitras\}@illinois.edu}}%
\thanks{$^{2}$Jingkai Chen and Chuchu Fan are with Massachusetts Institute of Technology, Cambridge, MA, 02139 USA.
        {\tt\footnotesize \{jk\_chen, chuchu\}@mit.edu}}%
\thanks{Digital Object Identifier (DOI): see top of this page.}
}

\markboth{IEEE Robotics and Automation Letters. Preprint Version. Accepted January, 2022}
{Sun \MakeLowercase{\textit{et al.}}: Multi-agent Motion Planning from Signal Temporal Logic Specifications}

\maketitle

\begin{abstract}
We tackle the challenging problem of 
multi-agent cooperative motion planning  for complex tasks described using signal temporal logic (STL), where robots can have nonlinear and nonholonomic dynamics.
Existing methods in multi-agent motion planning, especially those based on discrete abstractions and model predictive control (MPC), suffer from limited scalability with respect to the complexity of the task, the size of the workspace, and the planning horizon. 
We present a method based on {\em timed waypoints\/} to address this issue. 
We show that timed waypoints can help abstract nonlinear behaviors of the system as safety envelopes around the reference path defined by those waypoints. Then the search for waypoints satisfying the STL specifications can be inductively encoded as a mixed-integer linear program.
The agents following the synthesized timed waypoints have their tasks automatically allocated, and are guaranteed to satisfy the STL specifications while avoiding collisions. 
We evaluate the algorithm on a wide variety of benchmarks.
Results show that it supports multi-agent planning from complex specification over long planning horizons, and significantly outperforms state-of-the-art abstraction-based and MPC-based motion planning methods. The implementation is available at \url{https://github.com/sundw2014/STLPlanning}.
\end{abstract}


\section{Introduction}
\label{section:intro}

\IEEEPARstart{T}{he} capability of performing automatic task and motion planning according to high-level specifications is what people usually expect from an intelligent and autonomous robotic system.
These high-level task specifications usually consist of temporal and logical rules and need cooperative solutions of multiple agents. It is not straightforward to directly derive a specific sequence of locations to visit for each agent from these high-level specifications.
Therefore, synthesizing correct-by-construction plans and control strategies from these complicated specifications has been an open problem, especially when the planning horizon is long and the robotic systems have complex dynamics~\cite{kress2018synthesis,majumdar2020multiparty}.
Fortunately, Temporal Logic (TL), especially Signal Temporal Logic (STL) provides a mathematically precise language for specifying tasks and rules over continuous signals with explicit time semantics~\cite{donze2010robust}.
Such a formal description of the task enables automatic control action synthesis.

Methods based on discrete abstractions and model predictive control (MPC) are two representative approaches for motion planning  from TL specifications. Abstraction-based methods discretize the state space and create an abstract bisimilar graph or automaton, on which the actual planning is performed. MPC-based methods discretize the trajectory with a fixed time step, and the states at each time step are viewed as the decision variables of an optimization problem. Although these methods have achieved  success in a wide range of applications, some obvious disadvantages prevents them from being widely adopted for solving realistic robotic planning problems: To use abstraction-based methods, one needs to construct the bisimilar graph, which heavily relies on domain expertise. Moreover, the number of abstracted states would potentially grow exponentially fast as the dimensionality of the state space increases and cause significant scalability issues. As for MPC-based methods, the number of time steps needed might be too large for long-horizon planning. In Sec.~\ref{sec:exp}, we compare the proposed method with those two methods.

In this paper, we propose a novel synthesis method which tackles the aforementioned challenges. Inspired by the method in~\cite{fan2020fast} where the authors use piece-wise linear (PWL) reference paths and tracking controllers to solve simple reach-avoid synthesis problems, we show that using PWL reference paths one can also handle more expressive STL specifications. A PWL path is defined by a sequence of time-stamped waypoints. Our method can automatically reason over the STL formula by recursively encoding constraints over the time-stamped waypoints according to the syntax of the STL formula. Moreover, we define the multi-agent STL, which can be used to specify tasks that need be completed cooperatively by a group of agents. Given such a multi-agent STL formula, our proposed method can automatically assign sub-tasks to each agent such that they cooperate efficiently without collision. Because the tracking error of the tracking controller is taken into account when encoding the constraints, it can be shown that any solution that satisfy our encoded constraints can give the desired PWL paths: Agents following the PWL reference paths are guaranteed to satisfy the given multi-agent STL specification and are collision-free. More importantly, our constraints are all linear because of the PWL structure. Therefore, we can find optimal solutions by solving  mixed-integer linear programming (MILP) problems, which can be effectively handled by off-the-shelf solvers such as Gurobi\footnote{https://www.gurobi.com/}.

We evaluate the proposed method on 8 benchmark synthesis problems with a variety of different scenarios. We compare with both abstraction-based and MPC-based methods~\cite{leung2020back,raman2014model}. Empirical results show that our method outperforms other state-of-the-art methods in terms of running time and quality of the planned paths, not to say that our method can handle much more general STL formulas. For example, our method is order of magnitude faster than \texttt{stlcg}~\cite{leung2020back}, which is MPC-based and solves the optimization problem with gradient decent. Also, we implement and compare with another MPC-based algorithm proposed in~\cite{raman2014model}. The MPC-based method failed in some cases due to the large number of decision variables, while our method successfully found a solution.

\subsection{Related work}
Robot motion planning is a large and active research area~\cite{lavalle2006planning,latombe2012robot,ghosh2020koord,ghosh2020cyphyhouse}, and planning from TL specifications has received significant attention~\cite{kress2018synthesis,tellex2020robots}.
Abstraction-based approaches have stood out as a systematic framework of finding control policies~\cite{kress2018synthesis, FainekosGKP09}.
However, the abstraction step of these approaches heavily relies on domain expertise and is hard to be automated. Among all the planning methods that can handle TL specifications, the closest to ours is the one proposed in~\cite{buyukkocak2021planning}, which discretizes the workspace into regions, constructs a graph with regions as nodes, and finally searches for a valid path on the graph.

\revise{Another class of synthesis approaches for STL is based on model predicative control (MPC), for example,~\cite{raman2014model,farahani2015robust,sadraddini2015robust,sahin2019multirobot,JoLeVaSaSeTrBe-ISRR-2019}.} In these approaches, a time-step is fixed and the decision variables of the optimization are just the state at each step. Both the dynamics and the STL specifications are encoded as constraints of the optimization problem. Thus, it is challenging to handle real-world robots with complicated dynamics. Another drawback of these approaches is that the number of time steps needed might be too large for long-horizon planning. \revise{The proposed approach tackles this problem by using time-stamped waypoints instead of a fixed time step, which is also similar to event-based control (e.g.,~\cite{gundana2021event}) in the sense that each waypoint can be viewed as an event and between two consecutive events the control input does not change. Similar idea has been studied in~\cite{yang2020continuous}, where the users use zeroth-order hold control, i.e., the control signal is held at a time instant (waypoint) for a variable interval. Different from the proposed approach, it uses control barrier functions to ensure satisfaction between timed waypoints.}

Sampling-based methods have also been used to solve planning problems for multi-agent systems and/or STL specifications. \revise{STyLuS$^*$~\cite{kantaros2020stylus} is a scalable algorithm for multi-agent optimal control with temporal logics.} \cite{vasile2020reactive} utilizes RRT to plan paths for long-term LTL goals with short-term reactive specifications. In~\cite{KARLSSON202015537}, the authors proposes the spatio-temporal RRT* algorithm which can handle STL specifications containing only ``always" operators. \revise{In~\cite{vasile2017sampling}, the authors extend the RRT$^*$ algorithm with biased space-time sampling and guided steering, and the algorithm is able to efficiently grow the RRT tree along the direction of increasing STL satisfaction.}

\section{Preliminaries and Problem Statement}
\label{section:problem}

Let $\mathbb{R}$ and $\mathbb{Z}^+$ be the real numbers and positive integers respectively. For a vector $x \in \mathbb{R}^n$, $x^{(i)}$ is its $i^\text{th}$ entry, $\norm{x}$ is its Euclidean norm, $\|x\|_1$ is its one-norm, and $B_\epsilon(x) := \{y \in \mathbb{R}^n\ | \ \norm{y-x} \leq \epsilon\}$ is the $\epsilon$-ball centered at $x$. Given a matrix $H \in \mathbb{R}^{n \times m}$ and a vector $b \in \mathbb{R}^n$, $\poly(H, b)$ denotes the convex polytope $\{x \in \mathbb{R}^m\ | \ H \cdot x \leq b\}$. $H^{(i)}$ is the $i^\text{th}$ row of $H$, and $\polydim{H}$ denotes the number of rows in $H$, which is also the number of faces of the polytope. For $N \in \mathbb{Z}^+$, denote $\{1, \cdots, N\}$ by $[N]$.

\subsection{STL for multi-agent specifications}
\label{sec:STL-def}
\label{section:problem:stl}
Let $\mathcal{W} := \reals^d$ be the workspace. Given a \revise{vector-valued} function $\mu$ defined on $\mathcal{W}$, an {\it atomic predicate} can be defined based on $\mu$ and is denoted by $\pi^\mu$. For a point $x \in \mathcal{W}$, we say that $x$ satisfies $\pi^\mu$ (written as $x \sat \pi^\mu$) iff. $\mu(x) \geq 0$. In this paper, we are only interested in atomic predicates that indicate whether or not a point is in a polytope. That is, $\mu$ is always of the form $\mu(x) = b - H \cdot x$. Then, $x \sat \pi^\mu$ iff. $x \in \poly{(H,b)}$. \revise{Also, $x \nvDash \pi^\mu$ iff. $x \notin \poly{(H,b)}$.} Atomic predicates only characterize standalone points in the workspace. However, we are more interested in predicates that can characterize trajectories. Let $p : \nnreals \mapsto \mathcal{W}$ be the position trajectory of a robot, which is a function of time. \revise{Let $(p, t)$ be the suffix of $p$ at $t$, i.e., $(p, t)(s) = p(s+t)$.} Next, STL is defined based on the atomic predicates.
\begin{definition}[Signal Temporal Logic (STL)] An STL formula is defined with the following syntax:
\begin{equation}
\begin{aligned}
\varphi::=\ & \pi^{\mu} | \neg \pi^{\mu} | \varphi_{1} \wedge \varphi_{2} | \varphi_{1} \vee \varphi_{2} \\
            & | \eventually_{[a, b]} \varphi | \always_{[a, b]} \varphi | \varphi_{1} \mathcal{U}_{[a, b]} \varphi_{2} | \varphi_{1} \release_{[a, b]} \varphi_{2}
\end{aligned}
\end{equation}
where $\varphi, \varphi_{1}, \varphi_{2}$ are STL formulas, and $0 \leq a \leq b < \infty$ denote time intervals. Here, the temporal operators $\eventually, \always, \until, \release$ are called ``eventually", ``always", ``until", and ``release" respectively. 
Formally, \revise{the validity of an STL formula with respect to a trajectory} $p : \nnreals \mapsto \mathcal{W}$ is defined as follows.
\begin{align*}
& ~p \vDash \varphi & \Leftrightarrow & \quad\left(p, 0\right) \vDash \varphi\\
& \left(p, t\right) \vDash \pi^{\mu} \quad & \Leftrightarrow & \quad \mu\left(p(t)\right) \geq 0\\
& \left(p, t\right) \vDash \neg \pi^{\mu} & \Leftrightarrow & \quad \revise{\left(p, t\right) \nvDash \pi^{\mu}}\\
& \left(p, t\right) \vDash \varphi_1 \wedge \varphi_2 \quad & \Leftrightarrow & \quad\left(p, t\right) \vDash \varphi_1 \wedge\left(p, t\right) \vDash \varphi_2\\
& \left(p, t\right) \vDash \varphi_1 \vee \varphi_2 \quad & \Leftrightarrow & \quad\left(p, t\right) \vDash \varphi_1 \vee \left(p, t\right) \vDash \varphi_2\\
& \left(p, t\right) \vDash \eventually_{[a, b]} \varphi \quad & \Leftrightarrow & \quad \exists t^{\prime} \in\left[t+a, t+b\right],\left(p, t^{\prime}\right) \vDash \varphi\\
& \left(p, t\right) \vDash \always_{[a, b]} \varphi \quad & \Leftrightarrow & \quad \forall t^{\prime} \in\left[t+a, t+b\right], \left(p, t^{\prime}\right) \vDash \varphi\\
& \left(p, t\right) \vDash \varphi_1 \mathcal{U}_{[a, b]} \varphi_2 & \Leftrightarrow & \quad\exists t^{\prime} \in\left[t+a, t+b\right], \left(p, t^{\prime}\right) \vDash \varphi_2
\\&&& \quad
\wedge \forall t^{\prime \prime} \in\left[t, t^{\prime}\right],\left(p, t^{\prime \prime}\right) \vDash \varphi_1\\
& \left(p, t\right) \vDash \varphi_1 \release_{[a, b]} \varphi_2 & \Leftrightarrow & \quad \forall t^{\prime} \in\left[t+a, t+b\right], \left(p, t^{\prime}\right) \vDash \varphi_2
\\ &&& \quad 
\vee \exists t^{\prime \prime} \in\left[t, t^{\prime}\right], \left(p, t^{\prime \prime}\right) \vDash \varphi_1
\end{align*}
\end{definition}

\revise{Please note that in the above syntax, negation can only be applied to atomic predicates. This is known as the {\it Negation Normal Form} and is not restrictive because any STL formula can be put in this form~\cite{lavalle2006planning}.}
The above definition of STL only characterizes a single trajectory $p$. Next, we define the multi-agent STL (MA-STL), which extends the notion of STL to cases where multiple trajectories are considered.
\begin{definition}[Multi-agent STL (MA-STL)]
An $N$-agent STL formula is defined recursively with the following syntax:
\[
\Psi ::= \pi^\varphi_i | \Psi_1 \wedge \Psi_2 | \Psi_1 \vee \Psi_2,
\]
where $\Psi_1, \Psi_2$ are $N$-agent STL formulas, and $\pi_i^\varphi$ assigns a single-agent STL $\varphi$ to agent $i$. Formally, the validity an MA-STL w.r.t. a group of trajectories $(p_1, \cdots, p_N)$ is as follows.
\begin{align*}
& (p_1, \cdots, p_N) \sat \pi^\varphi_i & \Leftrightarrow \quad & p_i \sat \varphi,\\
& (p_1, \cdots, p_N) \sat \Psi_1 \wedge \Psi_2 & \Leftrightarrow \quad & (p_1, \cdots, p_N) \sat \Psi_1
\\&&&
\text{ and } (p_1, \cdots, p_N) \sat \Psi_2,\\
& (p_1, \cdots, p_N) \sat \Psi_1 \vee \Psi_2 & \Leftrightarrow \quad & (p_1, \cdots, p_N) \sat \Psi_1
\\&&&
\text{ or\ \ \ } (p_1, \cdots, p_N) \sat \Psi_2.
\end{align*}
\label{def:MA-STL}
\vspace{-0.6cm}
\end{definition}

\begin{remark} Firstly, MA-STL enables implicit task assignment. For example, let $\{\mathcal{G}_i\}_{i=1}^{M}$ be $M$ goals, which are atomic predicates defining some polytopes in the workspace. The MA-STL formula $\Psi = \bigwedge_{j=1}^{M}\bigvee_{i=1}^{N} \pi_i^{\eventually_{[0,T]} \mathcal{G}_j}$ assigns tasks to the agents implicitly. That is, it does not assign specific tasks to each agent, but requires each goal to be visited by at least one agent. As will be shown in Sec.~\ref{sec:exp}, with the proposed planning algorithm, agents can figure out the optimal assignment automatically. \revise{Secondly, please also note that MA-STL is only a syntactic sugar in the sense that it is a subset of the STL formulas defined over the joint state space of the multi-agent system. In this subset, temporal operations can only be applied to a single agent at a time.}
\label{example:MA-STL}
\end{remark}

\subsection{Tracking controllers for the agents}
\revise{In practice, the robots used to complete a task usually have complicated and nonlinear dynamics, which makes it difficult to directly synthesize the correct control input for them. As a famous aphorism goes: ``all problems in computer science can be solved by another level of indirection", we exploit a separation of concerns that exists in the robot control synthesis problem so that complexity of specifications (tasks) and that of the dynamics can be dealt with separately. Specifically, we assume that a tracking controller is given for each agent such that it can track {\it any} reference path under {\it any} bounded disturbances with an known tracking error $\epsilon > 0$.\footnote{This is true when the reference paths satisfy the requirements of the controller, for example, the velocity is bounded.} That is, the distance between the actual position of the robot and the desired position on the reference path is always upper bounded by $\epsilon$. Many techniques can be used for obtaining such a tracking controller and the corresponding tracking error, for example, control Lyapunov functions~\cite{rodriguez2014trajectory} or control contraction metrics~\cite{Sun2020}. Here, the tracking controller is an abstraction (or the so-called ``indirection") layer that wraps the underlying dynamics such that the closed-loop system has a uniform behavior (characterized by the tracking error bound), and thus makes the design of motion planners easier.}

\revise{Obviously, controlled by the tracking controller, the actual trajectory of the agent will be in a tube centered at the reference path. If one can show that every trajectory in this tube satisfy the specification, then it can be guaranteed that the actual trajectory of the agent will satisfy the specification in the presence of any bounded disturbances. To this end, we define the robustness of trajectories.}

\begin{definition}[$\epsilon$-robust]
A group of trajectories $(p_1, \cdots, p_N)$ is said to be $\epsilon$-robust with respect to a property for some $\epsilon > 0$, if the property holds for all $(\hat{p}_1, \cdots, \hat{p}_N)$ satisfying $\sup_{t} \|\hat{p}_i(t) - p_i(t)\|_2 \leq \epsilon$, $\forall i \in [N]$.
\end{definition}

\subsection{The MA-STL motion planning problem}
\label{sec:stl-mamp}
Next, we define the multi-agent motion planning problem. Intuitively, the goal is to find a group of reference paths that are $\epsilon$-robust to a given MA-STL specification and free of inter-agent collisions.
Assume that $N$ agents are involved and $T$ is the time bound. Denote the size of the $i$-th agent by $s_i > 0$. That is, at position $p \in \mathcal{W}$, the agent is completely contained in a ball around it of radius $s_i$, i.e., $B_{s_i}(p)$. Then, the planning problem is defined as follows.

\begin{definition}[MA-STL Motion Planning]
A MA-STL motion planning problem is defined by a tuple
\begin{equation}
 \langle p^\texttt{init}_1, \cdots, p^\texttt{init}_N, \Psi \rangle,    \nonumber
\end{equation}
where $p^{\texttt{init}}_i \in \mathcal{W}$ is the initial position of agent $i$ and $\Psi$ is an MA-STL formula.
The problem is to find a group of reference paths $(p_1, \cdots, p_N)$ satisfying the following conditions:
\begin{enumerate}
    \item (Initial conditions) $p_i(0) = p^\texttt{init}_i$, $\forall i \in [N]$.
    \item (No inter-agent collisions) $\forall t \in [0, T]$, $\forall i, j \in [N]$ and $i \not = j$, $B_{s_i+\epsilon}(p_i(t)) \cap B_{s_j+\epsilon}(p_j(t)) = \emptyset$.
    \item (STL Satisfaction) $(p_1, \cdots, p_N)$ are $\epsilon$-robust w.r.t. $\Psi$.
\end{enumerate}
\label{def:STL-MAMP}
\end{definition}

Instead of searching for the reference paths among all possible functions of time, the proposed approach restricts its search space to piece-wise linear (PWL) paths. In the rest of the paper, we refer to the reference PWL path for agent $i$ as $S_i$. Formally, PWL paths are defined as follows.

\begin{definition}[Piece-wise Linear Path]
A piece-wise linear path $S_i$ in the workspace $\mathcal{W}$ is a function $S_i: \nnreals \rightarrow \mathcal{W}$ that maps a time instant $t$ to a position $S_i(t) \in \mathcal{W}$. It is constructed from a sequence of time-stamped waypoints $\{(t_{i,k}, p_{i,k})\}_{k=0}^{K_i}$ such that $S_i(t) = p_{i,k-1} + \frac{p_{i,k} - p_{i,k-1}}{t_{i,k} - t_{i,k-1}} (t - t_{i,k-1})$ for $t \in [t_{i,k-1},t_{i,k}]$. Here, $0=t_{i,0} \leq t_{i,1} \leq \cdots \leq t_{i,K}$ are the time stamps, and $(t_{i,k}, p_{i,k}) \in \nnreals \times \mathcal{W}$ is called the $k^\text{th}$ {\bf waypoint} of path $S_i$. The restriction of $S_i$ on the $k^\text{th}$ time interval $[t_{i,k-1},t_{i,k}]$ is called the $k^\text{th}$ {\bf segment} of $S_i$ and is denoted by $S^{(k)}_i$.
\end{definition}

\section{Solving the planning problems using MILP}
\label{sec:main_technique}
\noindent {\bf Overview of the approach.} We formulate the problem of finding PWL paths satisfying the specifications (inter-agent collision avoidance and STL satisfaction) as a constrained optimization problem and solve its mixed-integer linear programming (MILP) encoding using off-the-shelf optimizers such as Gurobi. The optimization problem is as follows:
\begin{equation}
\begin{gathered}
\min_{\mathcal{C}} \quad \mathcal{L}(\mathcal{C})\\
\textrm{s.t. $(S_1, \cdots, S_N)$ satisfy the conditions in Def.~\ref{def:STL-MAMP}.}
\end{gathered}
\label{eq:opt}
\end{equation}
where $\mathcal{C} := \bigcup_{i=1}^{N}\bigcup_{k=0}^{K_i}\{t_{i,k}, p_{i,k}\}$ is the set of variables representing the time stamps and waypoints on the PWL reference paths $(S_1, \cdots, S_N)$, and $\{K_i\}_{i=1}^{N}$ are constants. Here, $\mathcal{L}$ is a linear cost function. For example, one can minimize the total travel time, $\mathcal{L}(\mathcal{C}) = \sum_{i=1}^N t_{i,K_i}$. One can also minimize the makespan using $\mathcal{L}(\mathcal{C}) = T_{makespan}$ with extra linear constraints $T_{makespan} \geq t_{i,K_i}$, $i = 1, \cdots, N$.

In order to solve the above optimization problem as a MILP problem, the constraint in Eq.~\eqref{eq:opt} must be transformed into {\it a conjunction of linear constraints}, where each constraint is of the form $\texttt{LE} \geq 0$, and $\texttt{LE}$ is a linear expression of the decision variables. In addition to the aforementioned continuous variables $\mathcal{C}$, the decision variables of the MILP problem will contain another set of variables $\mathcal{B}$ that are the binary variables introduced when encoding logic relations. Also, this transformation must be sound, i.e., the feasible set defined by the linear constraints should be a subset of the original feasible set in Eq.~\eqref{eq:opt}.

In our approach, we first convert the original constraint in Eq.~\eqref{eq:opt} into a {\em linear constraint formula (LCF)}, which is a logic sentence of atomic formulas connected by conjunction or disjunction operators. The atomic formulas are of the form $\texttt{LE}_{\mathcal{C}} \geq 0$, where $\texttt{LE}_{\mathcal{C}}$ is a linear expression of the continuous variables $\mathcal{C}$. Then, the disjunctions in the LCF are eliminated using the big-M method, and binary variables $\mathcal{B}$ are introduced in this step. After eliminating all the disjunctions, the LCF becomes a conjunction of linear constraints.

This section is structured as follows. We first show how to transform the STL satisfaction and inter-agent collision avoidance into LCFs in Section~\ref{sec:STL} and Section~\ref{sec:mult} respectively. In Section~\ref{sec:overall}, we show the overall algorithm.

\subsection{Encoding MA-STL satisfactions with LCFs}
\label{sec:STL}
In this section, we consider the problem of encoding an MA-STL specification $\Psi$ with LCFs. Recall the syntax of MA-STL in Definition~\ref{def:MA-STL}. In an MA-STL formula $\Psi$, there are only conjunction and disjunction operations in addition to single-agent STL formulas. Thus, if we can find an LCF for each single-agent STL formula $\pi_i^{\varphi}$ in $\Psi$, then these LCFs can be directly combined with conjunctions and disjunctions to get the LCF for $\Psi$. Hence, in this section, we only consider the encoding of single-agent STL formulas, and the subscript $i$ is omitted for simplicity.
In conclusion, given an STL formula $\varphi$ and the tracking error $\epsilon > 0$, we aim at obtaining an LCF over the time-stamped waypoints $\{t_k, p_k\}_{k=0}^{K}$ such that if this LCF is true then the PWL path is $\epsilon$-robust to $\varphi$. 

Such an LCF can be constructed inductively. We will construct an LCF for each of the $K$ segments of $S$ and denote them by $z_i^\varphi$, $i = 0, 1, \cdots, K-1$. \revise{We want $z_i^\varphi$ to have a strong soundness property: $z_i^\varphi\text{ is true } \implies \forall t \in [t_{i}, t_{i+1}], (p, t) \sat \varphi$ for any trajectory $p$ deviating from $S$ up to the tracking error $\epsilon$, i.e., starting from any time point on the segment, $\varphi$ is satisfied robustly.} Once we obtain such LCFs for $\varphi$, $z_0^\varphi$ is just the LCF we ultimately want. Fortunately, LCFs with such a property can be encoded inductively starting from the atomic predicates and their negations.

For an atomic predicate $\varphi = \pi^\mu$ or its negation $\neg \pi^\mu$, where $\mu(x) := b - H \cdot x$, it is easy to construct $z_i^\varphi$ by shrinking or bloating $\poly{(H,b)}$ as follows.
{\small
\begin{multline}
 z_i^{\pi^\mu} = \bigwedge_{j=1}^{\polydim{H}}\Big( \big(b^{(j)} - H^{(j)} \cdot p_i - \epsilon \|H^{(j)}\|_2 \geq 0\big)\\ \wedge \big(b^{(j)} - H^{(j)} \cdot p_{i+1} - \epsilon \|H^{(j)}\|_2 \geq 0 \big) \Big);
\label{eq:atomic}
\end{multline}
\hrule
\begin{multline}
z_i^{\neg \pi^\mu} = \bigvee_{j=1}^{\polydim{H}}\Big( \big( H^{(j)} \cdot p_i - b^{(j)} - \epsilon \|H^{(j)}\|_2 \geq 0 \big)\\ \wedge \big( H^{(j)} \cdot p_{i+1} - b^{(j)} - \epsilon \|H^{(j)}\|_2 \geq 0 \big) \Big).
\label{eq:neg-atomic}
\end{multline}
}
\noindent It is easy to verify that the constructed $z$ formulas have the aforementioned soundness property. Intuitively, Eq.~\eqref{eq:atomic} requires both endpoints of the $i$-th segment of $S$ to be in the shrunk polytope, which is sufficient for the whole segment to be in the polytope. Eq.~\eqref{eq:neg-atomic} requires both endpoints to be on the outside of at least one face of the bloated polytope, which is also sufficient for the whole segment to be outside the polytope.

\revise{For non-atomic predicates, its $z$ formula will depend on the $z$ formulas of its sub-predicates, e.g., $z^{\always_{[a,b]}\varphi}$ depends on $z^{\varphi}$. The principle behind the design of the encoding rules is induction: we should guarantee that the aforementioned soundness property holds for the resulting $z$ formula if it holds for all the $z$ formulas of the sub-predicates (i.e., the $z$ formulas on the right-hand side of the below encoding rules).}

For conjunctions and disjunctions, it is simply
$z_i^{\varphi_1 \wedge \varphi_2} = z_i^{\varphi_1} \wedge z_i^{\varphi_2}$;
$z_i^{\varphi_1 \vee \varphi_2} = z_i^{\varphi_1} \vee z_i^{\varphi_2}$.

Temporal operators are handled as follows.

\begin{equation}
\small
{z_i^{\always_{[a,b]}\varphi}} = \bigwedge_{j=0}^{K-1}\left([t_{j}, t_{j+1}] \cap [t_{i}+a, t_{i+1}+b] \neq \varnothing \Rightarrow z_j^\varphi\right);
\label{eq:always}
\end{equation}
\hrule
{\small
  \setlength{\abovedisplayskip}{6pt}
  \setlength{\belowdisplayskip}{\abovedisplayskip}
  \setlength{\abovedisplayshortskip}{0pt}
  \setlength{\belowdisplayshortskip}{3pt}
\begin{multline}
z_i^{\eventually_{[a,b]}\varphi} = (t_{i+1} - t_i \leq b-a) \\ \wedge \bigvee_{j=0}^{K-1} \left([t_{j}, t_{j+1}] \cap [t_{i+1}+a, t_i+b] \neq \varnothing \wedge z_j^\varphi\right);
\label{eq:eventually}
\end{multline}
}%

\hrule
{\small
  \setlength{\abovedisplayskip}{6pt}
  \setlength{\belowdisplayskip}{\abovedisplayskip}
  \setlength{\abovedisplayshortskip}{0pt}
  \setlength{\belowdisplayshortskip}{3pt}
\begin{multline}
z_i^{\varphi_1 \until_{[a,b]} \varphi_2} = (t_{i+1} - t_i \leq b-a) \wedge \\ \bigvee_{j=0}^{K-1} \Big([t_{j}, t_{j+1}] \cap [t_{i+1}+a, t_i+b] \neq \varnothing \wedge z_j^{\varphi_2} \\ \wedge \bigwedge_{l=0}^{j} \left([t_{l}, t_{l+1}] \cap [t_{i}, t_{i+1}+b] \neq \varnothing \implies z_l^{\varphi_1}\right)\Big);
\label{eq:until}
\end{multline}
}%

\hrule
{\small
  \setlength{\abovedisplayskip}{6pt}
  \setlength{\belowdisplayskip}{\abovedisplayskip}
  \setlength{\abovedisplayshortskip}{0pt}
  \setlength{\belowdisplayshortskip}{3pt}
\begin{multline}
z_i^{\varphi_1 \release_{[a,b]} \varphi_2} = \bigwedge_{j=0}^{K-1} \Big(\big([t_{j}, t_{j+1}] \cap [t_{i}+a, t_{i+1}+b] \neq \varnothing \\ \Rightarrow z_j^{\varphi_2}\big) \vee \bigvee_{l=0}^{j-1} \left([t_{l}, t_{l+1}] \cap [t_{i+1}, t_{i+1}+b] \neq \varnothing \wedge z_l^{\varphi_1}\right)\Big).
\label{eq:release}
\end{multline}
}%

With the above rules of encoding, we can encode any STL formula as an LCF as follows. As mentioned earlier, $z_0^\varphi$ is what we ultimately want. In order to obtain $z_0^\varphi$, all of its dependencies on other $z$ formulas have to be resolved. Therefore, the algorithm runs recursively. The recursion stops at atomic predicates since they do not depend on any other $z$ formulas as shown in Eq.~\eqref{eq:atomic} and Eq.~\eqref{eq:neg-atomic}.

\revise{In order to prove the aforementioned soundness property, we proceed by induction. The base cases are the atomic predicates (Eq.~\eqref{eq:atomic}~and~\eqref{eq:neg-atomic}), for which we have provided some intuitions earlier. Then, the induction step has to be verified for each non-atomic predicate. The verification is straightforward but tedious. Here, we only verify the one for the ``$\always$" operation in Eq.~\eqref{eq:always}. The induction hypothesis is that the soundness property holds for all $z$ formulas on the RHS of Eq.~\eqref{eq:always}. Considering any trajectory $p$ deviating from $S$ up to $\epsilon$, by induction hypothesis, if $z_j^\varphi$ is true, then $\forall t \in [t_j, t_{j+1}]$, $(p,t) \sat \varphi$. For any $t \in [t_i, t_{i+1}]$ and any $t^\prime \in [t+a, t+b]$, we must have that $t^\prime \in [t_i+a, t_{i+1}+b]$. Now, assume that $z_i^{\always_{[a,b]}\varphi}$ is true. Let $j$ be such that $t^\prime \in [t_j, t_{j+1}]$. Then, $[t_i+a, t_{i+1}+b] \cap [t_j, t_{j+1}] \neq \varnothing$. According to the encoding, this implies that $z_j^\varphi$ is true. By induction hypothesis, we have that $(p, t^\prime) \sat \varphi$. To summarize, if $z_i^{\always_{[a,b]}\varphi}$ is true, then $\forall t \in [t_i,t_{i+1}]$, $\forall t^{\prime} \in \left[t+a, t+b\right]$, $\left(p, t^{\prime}\right) \vDash \varphi$, which is equivalent to say that $\forall t \in [t_i,t_{i+1}]$, $\left(p, t\right) \vDash \always_{[a,b]}\varphi$. Thus, we have proved the soundness property for $z_i^{\always_{[a,b]}\varphi}$. A complete proof can be found in \ifx\extended\undefined Appendix of~\cite{extended}\else Appendix\fi.}

\begin{remark}
\revise{With the above proof, it should be clear that although the encoding rules are stronger than we would need, i.e., $z^\phi_i$ encodes satisfaction over the entire segment, and thus make the problem harder to solve, it is indeed necessary. Otherwise, the induction does not hold.}
\end{remark}

\subsection{Encoding inter-agent collision avoidance with LCFs}
\label{sec:mult}
\newcommand{\safe}[1]{\mathtt{safe}\left(#1\right)}

In this section, we consider the problem of encoding the inter-agent collision avoidance with LCFs. Specifically, we aim at obtaining an LCF such that if this LCF is true, then at any time, the distance between any two agents is safe.
First, let us consider how to encode the specification that two time-stamped line segments are at least $\epsilon$ away from each other, which will be the building block for encoding the inter-agent collision avoidance specification. Consider two time-stamped line segments, $\mathtt{SEG}_1$ and $\mathtt{SEG}_2$.
Let the endpoints of $\mathtt{SEG}_1$ be $(t_{11}, p_{11})$ and $(t_{12}, p_{12})$. Similarly, $(t_{21}, p_{21})$ and $(t_{22}, p_{22})$ are the endpoints of $\mathtt{SEG}_2$. Define a function $\safe{}$ mapping them to an LCF as follows.
\begin{multline*}
\safe{\mathtt{SEG}_1, \mathtt{SEG}_2, \epsilon} := \left([t_{11}, t_{12}] \cap [t_{21}, t_{22}] = \varnothing\right) \\ \vee \Bigg(\left\|\frac{p_{11}+p_{12}}{2} - \frac{p_{21}+p_{22}}{2}\right\|_1 \geq \quad \\ \left\|\frac{p_{11}-p_{12}}{2}\right\|_1 + \left\|\frac{p_{21}-p_{22}}{2}\right\|_1 + \epsilon \sqrt{d}\Bigg),
\end{multline*}
where $d$ is the dimensionality of the workspace. \revise{Intuitively, if the above LCF is true, either of the following two conditions is true. \begin{inparaenum}\item the two segments are disjoint in the time dimension; or \item in the spatial dimension, the distance between the two centers is greater than the summation of the half-lengths of the two segments with a margin $\epsilon$, and thus they are disjoint.\end{inparaenum}}~Then, the specification that all the agents will not collide with each other is encoded as follows.
\begin{equation*}
z_{\texttt{inter}} = \bigwedge_{\substack{i, j = 1\\ i \neq j}}^{N} ~ \bigwedge_{\substack{k=1,\dots,K_i\\ l=1,\dots,K_j}} \safe{S_i^{(k)}, S_j^{(l)}, 2\epsilon + s_i + s_j},
\end{equation*}
which is, again, an LCF of the decision variables $\bigcup_{i=1}^{N}\bigcup_{k=0}^{K_i}\{t_{i,k}, p_{i,k}\}$. Recall that $s_i$ is the size of agent $i$. Please also note that we use $1$-norm instead of $2$-norm in the encoding to make the resulting expression linear (or at least piece-wise linear). Also, a formal proof of soundness can be found in \ifx\extended\undefined Appendix of~\cite{extended}\else Appendix\fi.

\subsection{Overall algorithm}
\label{sec:overall}
In this section, we show the overall algorithm. Each step of the algorithm is explained in the following.

\noindent{\bf Construct an AND-OR tree.} In Section~\ref{sec:STL} and Section~\ref{sec:mult}, we have shown how to transform the STL satisfaction and inter-agent collision avoidance to LCFs. These LCF formulas can be further merged with conjunctions and disjunctions into a single LCF. Such an LCF can be represented as an AND-OR tree. There are three types of nodes in the tree, AND nodes (i.e., conjunctions), OR nodes (i.e., disjunctions), and leaf nodes. Each AND or OR node has a finite number of children. Each leaf node refers to a linear expression \texttt{LE}.

\noindent{\bf Additional constraints.} First, we need additional constraints for the time instants. For each PWL path $S_i$, we need $0 = t_{i,0} \leq t_{i,1} \leq \cdots \leq t_{i,K_i} \leq T$, where $T$ is a constant specified by the user. Secondly, the maximum velocity of the PWL paths should also be constrained. For each PWL path $S_i$,
\[
\|p_{i,k+1} - p_{i,k}\|_1 \leq \texttt{vmax} * (t_{i,k+1} - t_{i,k}),\,k = 0,1,\cdots,K_i-1,
\]
where $\texttt{vmax}$ is a constant specified by the user. Finally, the PWL path must start from the initial position of the agent, i.e., $p_{i,0} = p^{\texttt{init}}_i$.
Please note that all these constraints are linear and can be easily merged into the AND-OR tree.

\noindent{\bf Create MILP constraints from the AND-OR tree.} In order to create MILP constraints, we have to eliminate all the disjunctions in the tree so that the whole tree is converted into a conjunction of linear constraints, i.e., $\bigwedge_i \texttt{LE}_i \geq 0$. Then, we can add each $\texttt{LE}_i \geq 0$ as a linear constraint to the MILP optimizer. To eliminate the disjunctions, we use the big-M method. For example, given an OR node, $\bigvee_{i=1}^{n} \texttt{LE}_i \geq 0$, we introduce $n$ binary variables $z_i,~i=1,\cdots,n$. Then, it can be shown that the conjunctive form $\left(\bigwedge_{i=1}^{n} \texttt{LE}_i + (1 - z_i) \cdot M \geq 0\right) \wedge \left(\sum_{i=1}^{n}z_i \geq 1\right)$ is equivalent to the original disjunctive form, where $M$ is a large enough positive constant. Intuitively, if $z_i = 1$, then $\texttt{LE}_i + (1 - z_i) \cdot M \geq 0$ becomes the original constraint $\texttt{LE}_i \geq 0$. On the other hand, if $z_i = 0$, then $\texttt{LE}_i \geq 0$ is disabled since $\texttt{LE}_i + M \geq 0$ is trivially true regardless of the value of $\texttt{LE}_i$. Finally, $\sum_{i=1}^{n}z_i \geq 1$ enforces that at least one of the constraints is enabled.

\noindent{\bf Putting it all together.} The algorithm first creates continuous variables in the optimizer, which represents the waypoints. Then, it constructs the AND-OR tree of the linear constraints. Next, disjunctions in the tree are eliminated, and the tree is converted into a list of linear constraints, which are then added to the optimizer. After the optimizer finds a feasible solution, the values of the continuous variables are returned, which determine the PWL paths.

\noindent\revise{{\bf Complexity.} The computational cost of solving a MILP problem is mostly determined by the number of binary variables. Therefore, we analyze the number of binary variables introduced for encoding an STL formula $\varphi$ with respect to a PWL path of length $K$. Due to the use of the big-M method, each child of an OR node in the AND-OR tree introduces a binary variable. As in Eq.~(\ref{eq:eventually}-\ref{eq:release}), each $z$ formula consists of $\mathcal{O}(K)$ disjunctions.\footnote{Please note that in the encoding of $\release$, although there are $\mathcal{O}(K^2)$ disjunctions for a single segment, after merging the repeated ones, there are only $\mathcal{O}(K)$ necessary disjunctions.} Since each segment has a $z$ formula, we will have $\mathcal{O}(K^2)$ disjunctions in order to encode a single operation. Let $|\varphi|$ be the number of operators in $\varphi$. The complexity of the proposed approach is $\mathcal{O}(K^2 \cdot |\varphi|)$. For MPC-based methods (e.g.,~\cite{raman2015reactive}), the complexity of encoding is $\mathcal{O}(N \cdot |\varphi|)$, where $N$ is the number of time steps. Although the proposed approach has a quadratic complexity while the MPC-based approach has a linear complexity, in many practical cases, the time horizon is long (and hence $N$ is large) but the task can be completed with very few line segments. In these cases, $K^2 \ll N$ and the proposed method drastically outperforms MPC-based methods, which is empirically verified in the experiments in Sec.~\ref{sec:exp}. On the other hand, in cases where the time horizon is small but the required number of segments is somehow large, using a MPC-based approach could be a better choice.}

\section{Experimental evaluation}
\label{sec:exp}
We evaluate the proposed approach on several benchmark scenarios and compare it with several other methods. The algorithm is implemented in Python, and the Gurobi optimizer is used for solving the MILP problems. The implementation is available at \url{https://github.com/sundw2014/STLPlanning}.
As for the robot dynamics, we use the Dubins vehicle model.
All experiments were conducted on a Linux workstation with two Intel Xeon Silver 4110 CPUs and 32 GB RAM.

\begin{figure*}[tbp]
    \centering
    \begin{subfigure}{.16\textwidth}
      \centering
      \includegraphics[width=\linewidth]{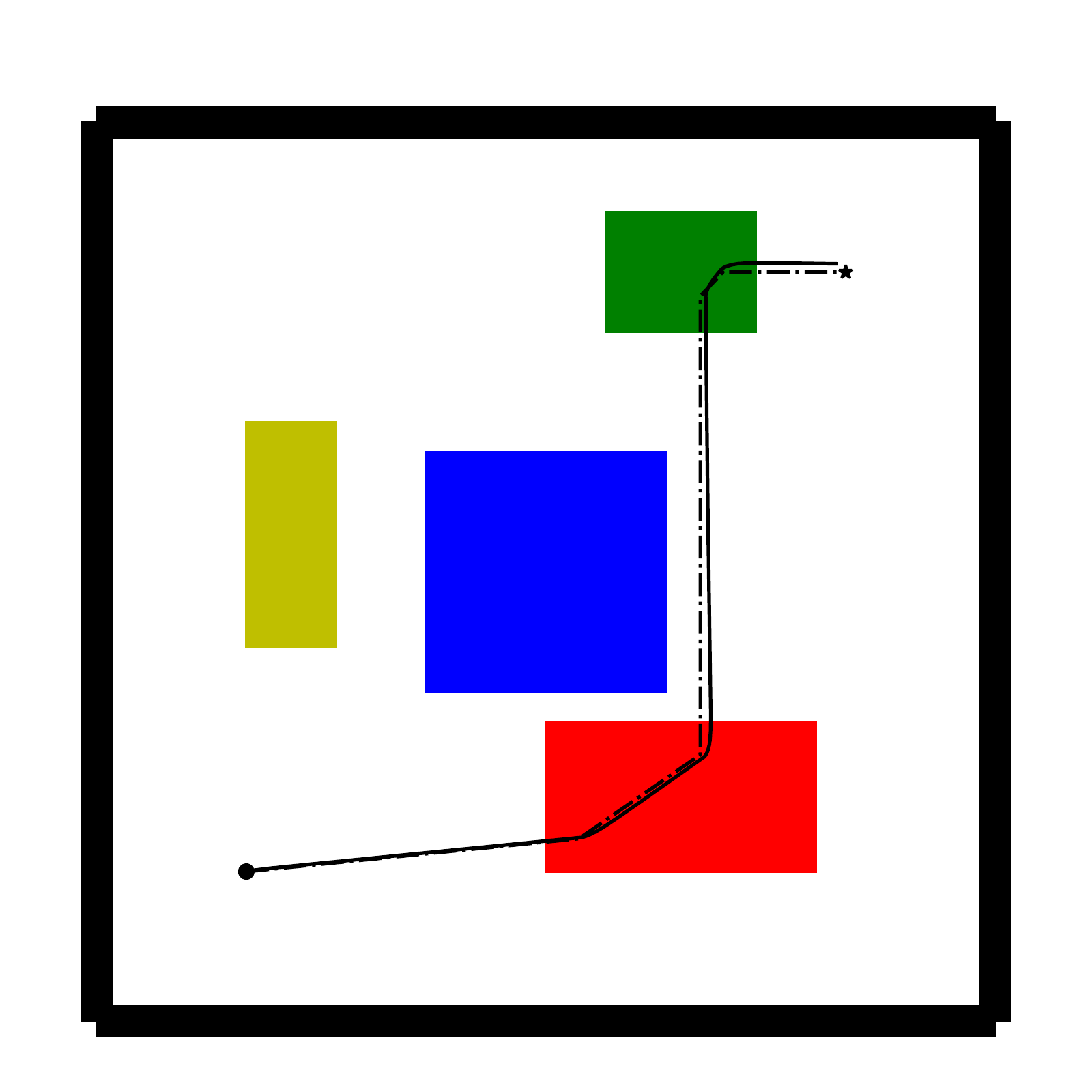}
      \caption{\texttt{stlcg-1}}
      \label{fig:stlcg-1}
    \end{subfigure}
    \begin{subfigure}{.16\textwidth}
      \centering
      \includegraphics[width=\linewidth]{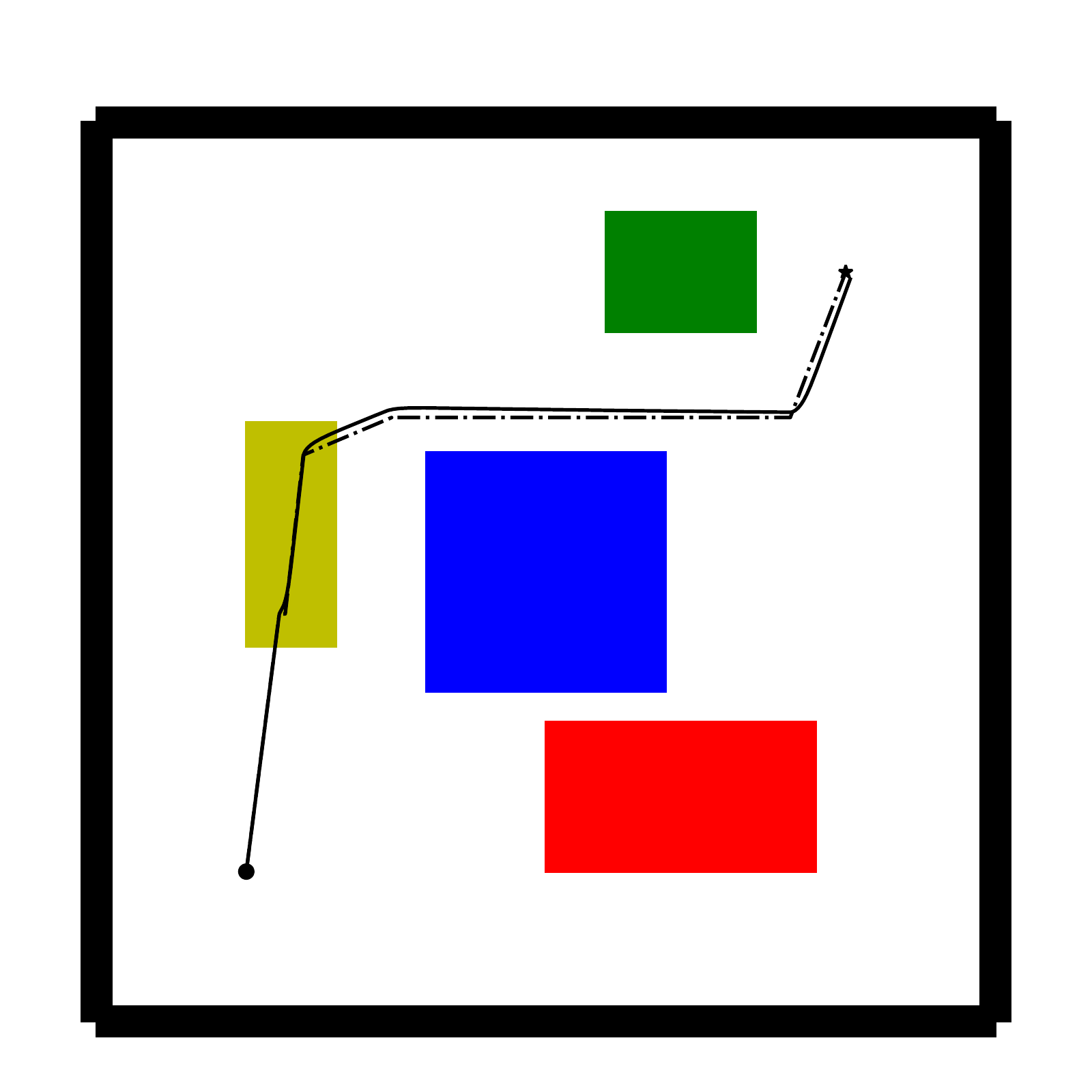}
      \caption{\texttt{stlcg-2}}
      \label{fig:stlcg-2}
    \end{subfigure}
    \begin{subfigure}{.16\textwidth}
      \centering
      \includegraphics[width=\linewidth]{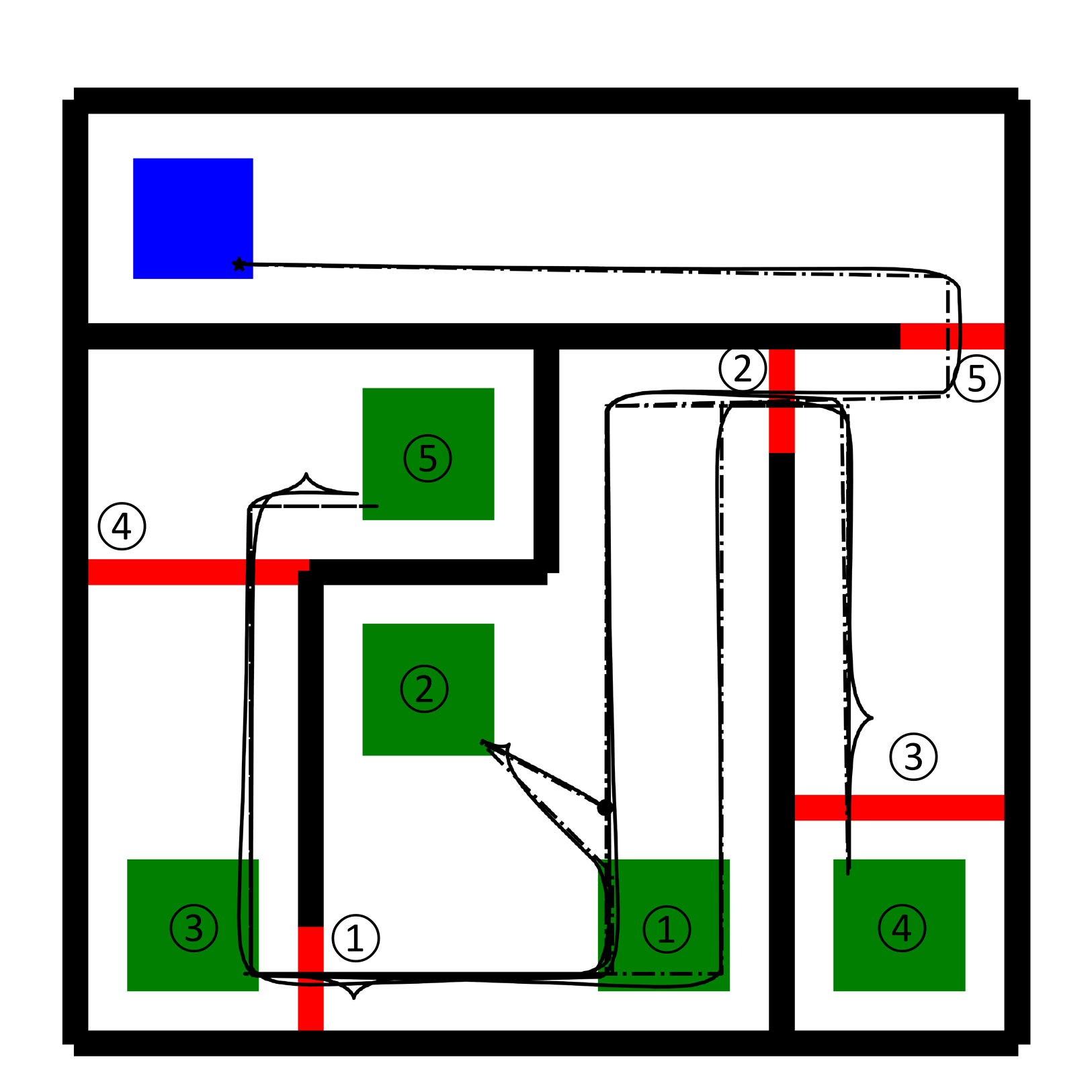}
      \caption{\texttt{doorpuzzle-1}}
      \label{fig:doorpuzzle-1}
    \end{subfigure}
    \begin{subfigure}{.16\textwidth}
      \centering
      \includegraphics[width=\linewidth]{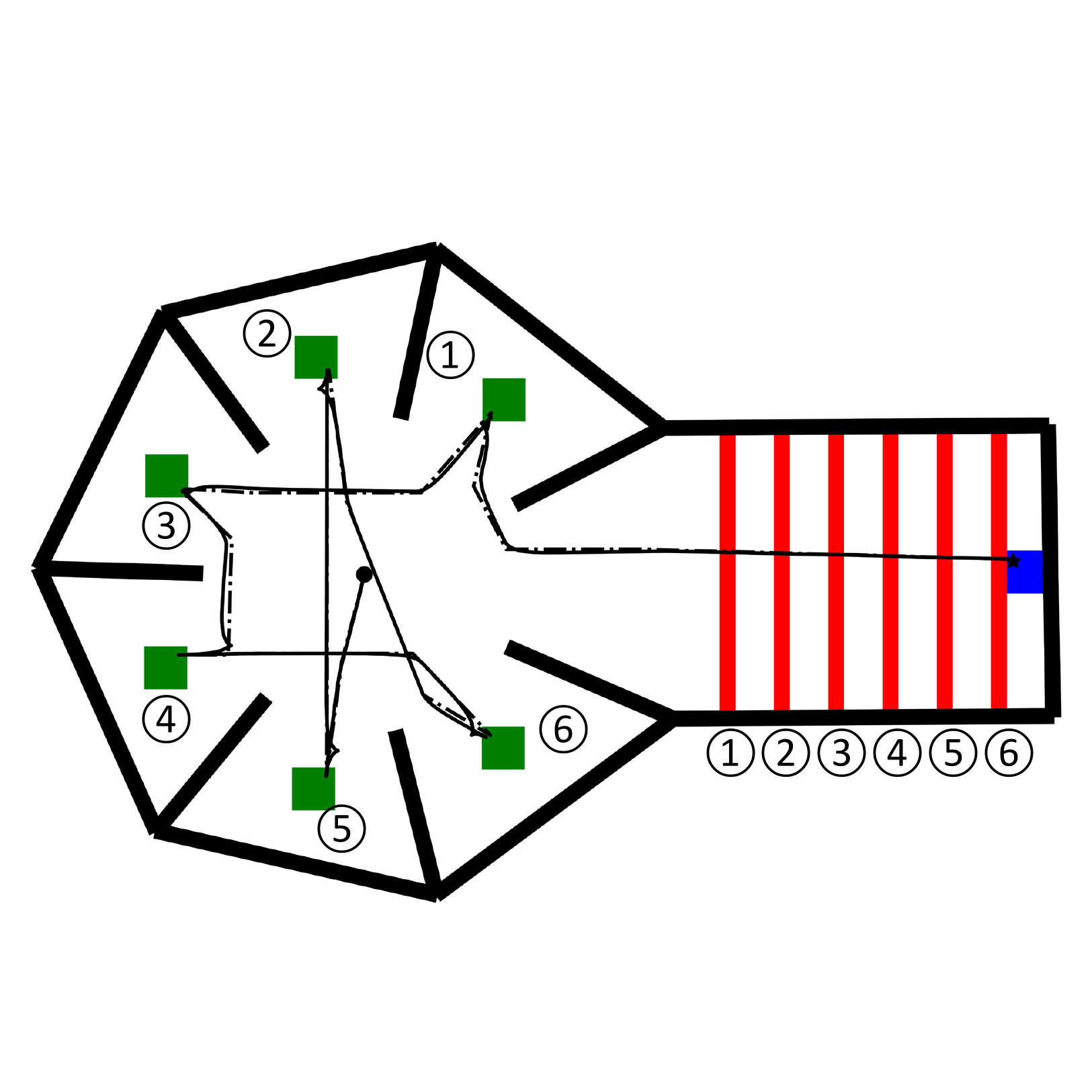}
      \caption{\texttt{doorpuzzle-2}}
      \label{fig:doorpuzzle-2}
    \end{subfigure}
    \begin{subfigure}{.16\textwidth}
      \centering
      \includegraphics[width=\linewidth]{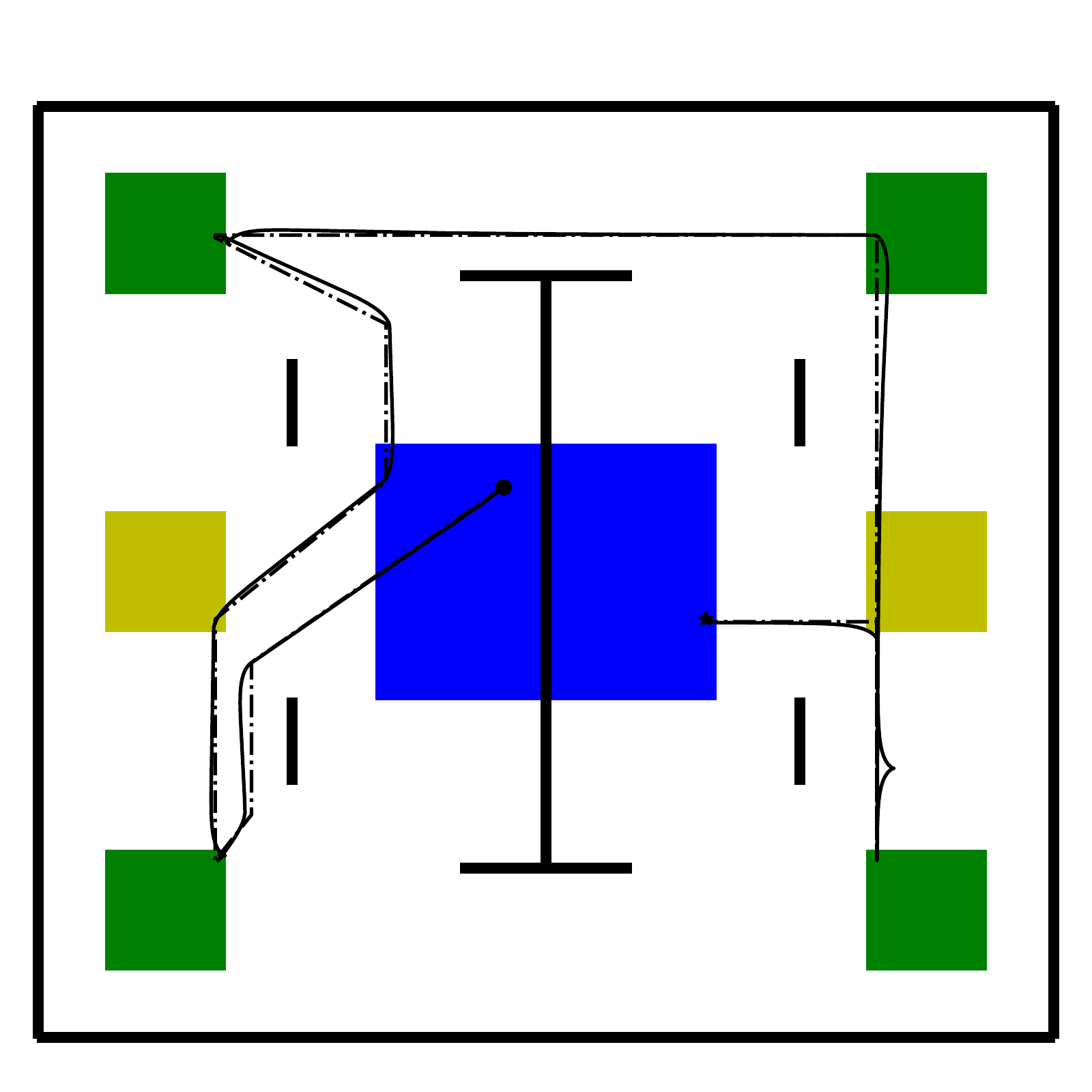}
      \caption{\texttt{rover-1}}
      \label{fig:rover-1}
    \end{subfigure}
    \begin{subfigure}{.16\textwidth}
      \centering
      \includegraphics[width=\linewidth]{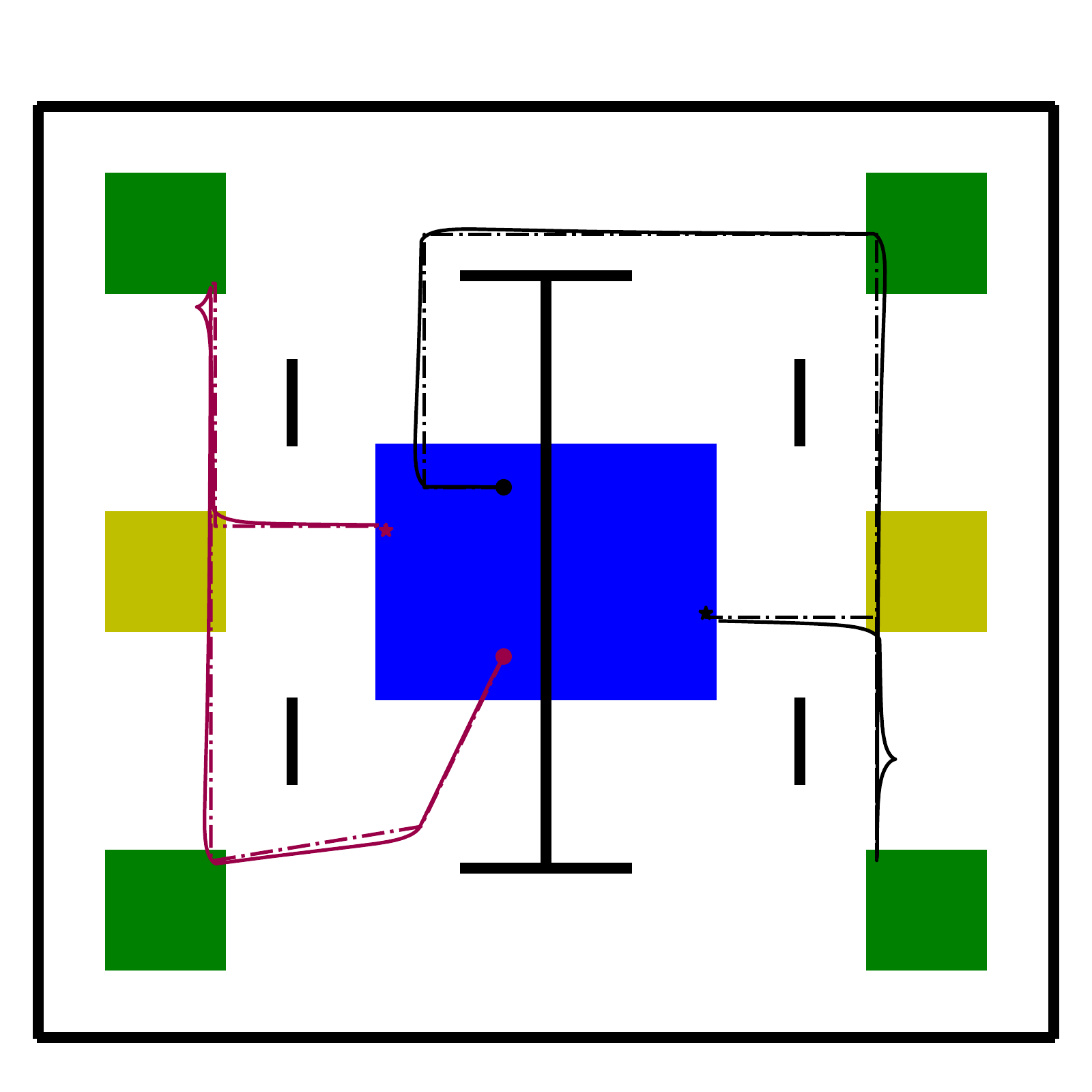}
      \caption{\texttt{rover-2}}
      \label{fig:rover-2}
    \end{subfigure}
    \caption{Benchmarks and results. Dashed-lines are the PWL paths found by the proposed method; Solid lines are the actual trajectories tracking the PWL paths; The circle on the trajectory is the starting point, and star is the end.}
    \label{fig:benchmarks-1}
\end{figure*}

\subsection{Benchmarks}
\label{sec:benchmarks}
Some of the benchmarks were borrowed from the motion planning literature~\cite{vega2018admissible, leung2020back}. We also designed several other benchmarks in order to show the ability of the proposed method to handle complicated MA-STL specifications. Specifically, the following benchmarks are used.

\texttt{stlcg-1} is from ~\cite{leung2020back}. As shown in Fig.~\ref{fig:stlcg-1}, a robot starting from the bottom-left corner is asked to visit the up-right corner. It is also asked to visit and avoid some regions in the middle. Denote the four regions by $Y$ (yellow), $B$ (blue), $G$ (green), and $R$ (red) respectively. The task is specified using an STL formula $\left(\eventually_{[0,T]}\always_{[0,5]}R\right) \wedge \left(\eventually_{[0,T]}\always_{[0,5]}G\right) \wedge \left(\always_{[0,T]} \neg B\right)$. Please note that in the original benchmark, region $B$ is a circle, and here we replace it with its circumscribed square. \texttt{stlcg-2} uses the same environment as in \texttt{stlcg-1} but with a different STL specification $\left(\eventually_{[0,T]}\always_{[0,5]}Y\right) \wedge \left(\always_{[0,T]} \neg G\right) \wedge \left(\always_{[0,T]} \neg B\right)$.

\texttt{doorpuzzle-1} is from ~\cite{vega2018admissible}. As shown in Fig.~\ref{fig:doorpuzzle-1}, a robot is asked to visit the goal region (blue). However, there are walls (black) and doors (red) in the environment. Before being able to open a door, the robot has to visit the correspondingly numbered green region to pick the key. Denote the goal by $G$, the wall by $W$, the doors by $D_1, \cdots, D_5$, and the keys by $K_1, \cdots, K_5$. Then, the task can be specified using an STL formula $\left(\eventually_{[0,T]} G\right) \wedge \left(\always_{[0, T]} \neg W\right) \wedge \left(\bigwedge_{i=1}^{5}\neg D_i \until_{[0,T]} K_i\right)$. \noindent\texttt{doorpuzzle-2} is a similar scenario from ~\cite{vega2018admissible} with 6 doors.

\texttt{rover-1} and \texttt{rover-2} are designed to evaluate the ability to handle complicated multi-agent STL specifications. As shown in Fig.~\ref{fig:rover-1}, rovers are asked to visit the goal regions (green) to make scientific observations while conforming to the following rules: \begin{inparaenum} \item Every rover should visit the charging station (blue) within $t_c$ time units every time they leave the charging station; \item After visiting a goal region, the rover should visit a transmitter (yellow) within $t_d$ time units, to transmit the collected data to the remote control; \item The rovers should avoid the walls (black) and each other. \end{inparaenum} Denote the charging station by $C$, the walls by $W$, the transmitters by $S_1$ and $S_2$, and the goals by $G_1, \cdots, G_4$. Then, the rule of charging can be encoded as $\varphi_1 := \always_{[0,T]}(\neg C \implies \eventually_{[0,t_{c}]} C)$. The rule of transmitting can be encoded as $\varphi_2 := \always_{[0,T]}(\bigvee_{i=1}^{4} G_i \implies \eventually_{[0,t_d]} \bigvee_{i=1}^{2} S_i)$. The rule of avoiding walls can be encoded as $\varphi_3 := \always_{[0,T]}\neg W$. Assuming that $N$ rovers are involved, the MA-STL specification encoding the task is $\Psi = \bigwedge_{i=1}^{N} \pi_i^{\left(\varphi_1 \wedge \varphi_2 \wedge \varphi_3\right)} \wedge \bigwedge_{j=1}^{4}\bigvee_{i=1}^{N} \pi_i^{\eventually_{[0,T]} G_j}$. We set $N = 1$ and $2$ for \texttt{rover-1} and \texttt{rover-2} respectively.

\texttt{wall-1} and \texttt{wall-2} are designed to evaluate the ability to arrange multiple agents to avoid collisions. As shown in Fig.~\ref{fig:wall-1} and Fig.~\ref{fig:wall-2}, a group of agents are asked to visit some goal regions, but there is a narrow door in the middle of the map. In order to avoid collisions, the agents have to figure out an order for them to go through the door. Let the wall (black) be $W$, and the goals be $G_1, \cdots, G_4$. The task is specified as $\Psi = \bigwedge_{i=1}^{4} \pi_i^{\always_{[0,T]} \neg W \wedge \eventually_{[0,T]} G_i}$.

\begin{figure*}[tbp]
    \centering
    \begin{subfigure}{\textwidth}
        \centering
        \includegraphics[width=0.23\textwidth, trim={0 120 0 120}, clip]{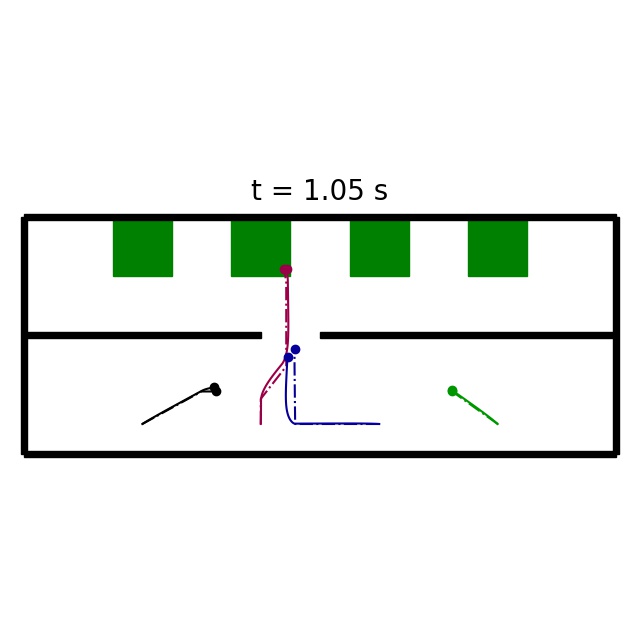}
        \includegraphics[width=0.23\textwidth, trim={0 120 0 120}, clip]{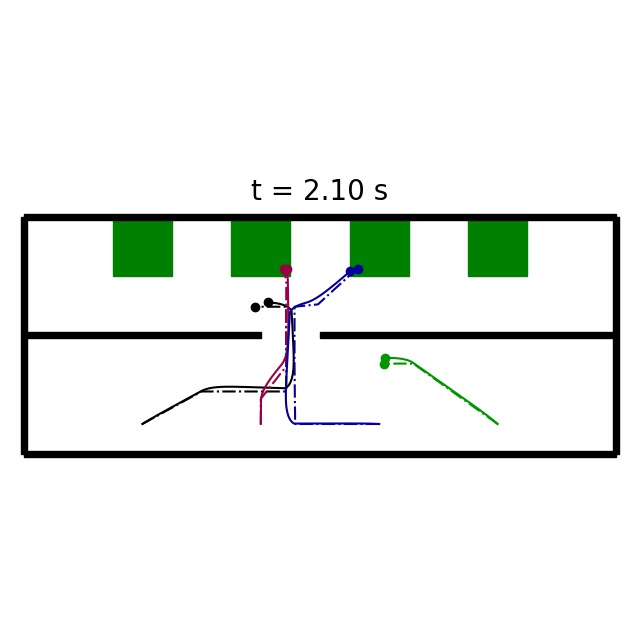}
        \includegraphics[width=0.23\textwidth, trim={0 120 0 120}, clip]{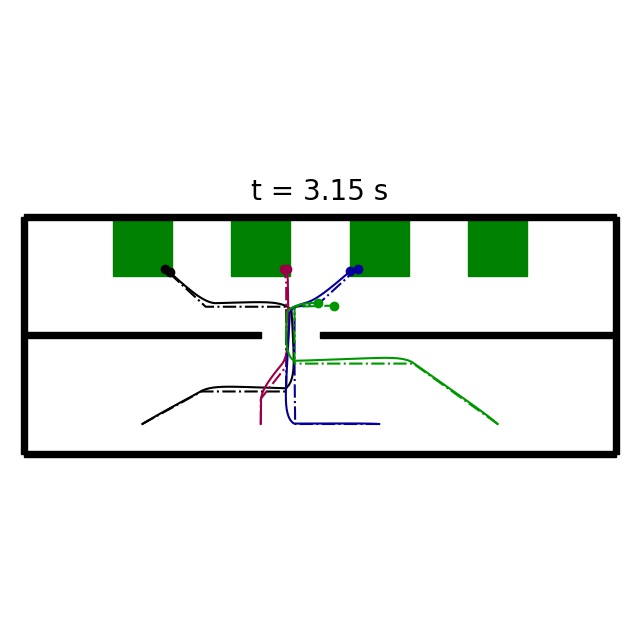}
        \includegraphics[width=0.23\textwidth, trim={0 120 0 120}, clip]{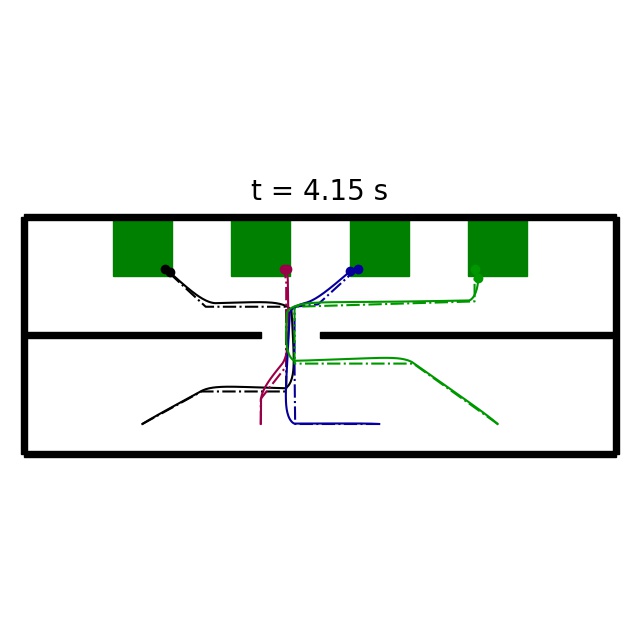}
        \caption{\texttt{wall-1}}
        \label{fig:wall-1}
    \end{subfigure}
    \begin{subfigure}{\textwidth}
        \centering
        \includegraphics[width=0.23\textwidth, trim={0 120 0 120}, clip]{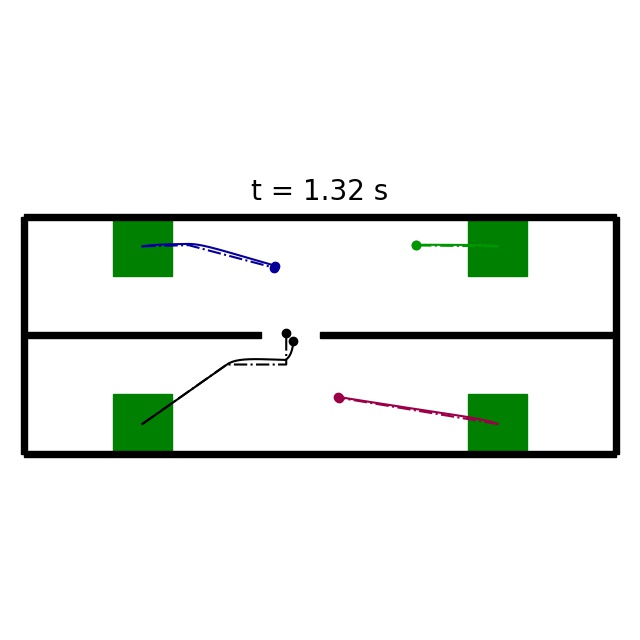}
        \includegraphics[width=0.23\textwidth, trim={0 120 0 120}, clip]{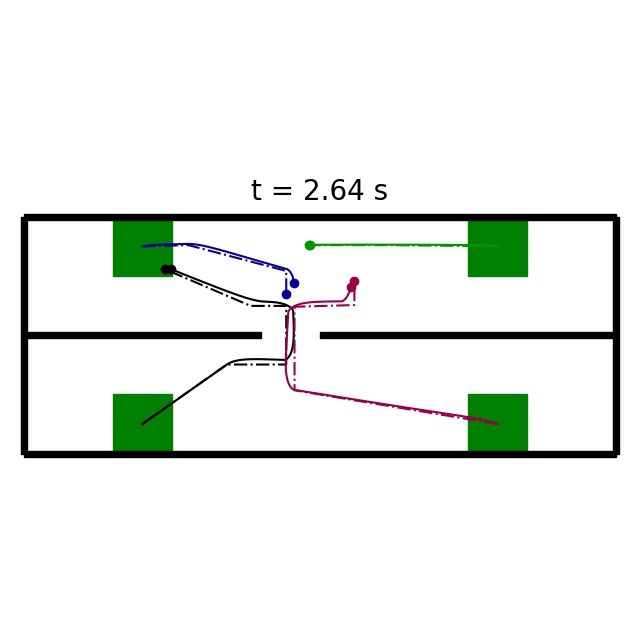}
        \includegraphics[width=0.23\textwidth, trim={0 120 0 120}, clip]{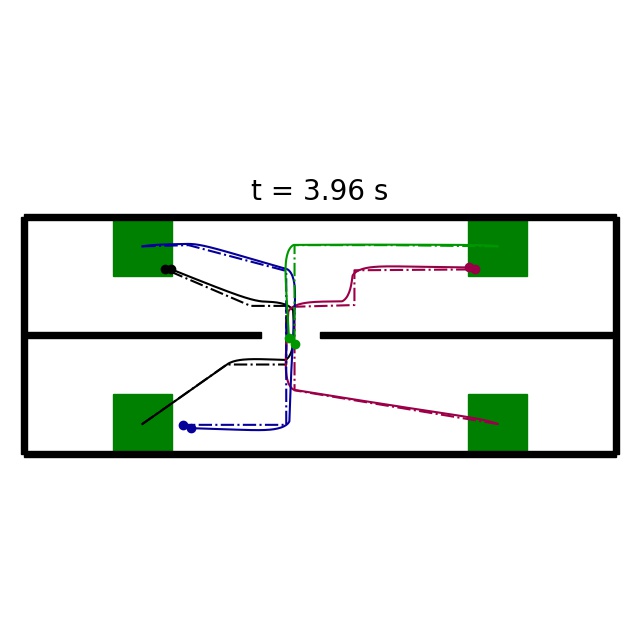}
        \includegraphics[width=0.23\textwidth, trim={0 120 0 120}, clip]{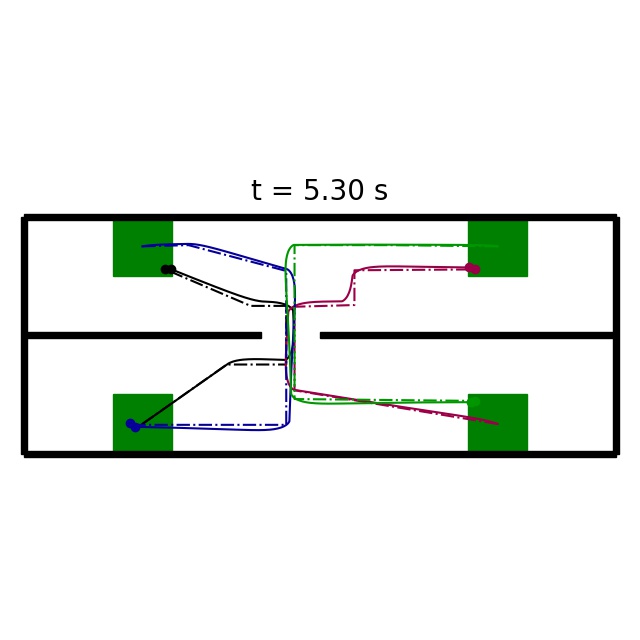}
        \caption{\texttt{wall-2}}
        \label{fig:wall-2}
    \end{subfigure}
    \caption{Benchmarks and results. Dashed-lines are the PWL paths found by the proposed method; Solid lines are the actual trajectories tracking the PWL paths. For each benchmark, we show four snapshots of the simulation with a clock in the title. The dots indicate the current locations of the agents, and agents are marked in different colors.}
    \label{fig:benchmarks-2}
\end{figure*}

\subsection{Comparison with other methods}
We compared our method with others, including an MPC-based method~\cite{raman2014model} and an abstraction-based method. Please note that although~\cite{buyukkocak2021planning} is the closest work to ours, the authors did not provide a publicly available implementation of their approach. Also, the benchmarks \texttt{doorpuzzle-1} and \texttt{doorpuzzle-2} are borrowed from~\cite{vega2018admissible}, but the authors did not either provide an implementation of their algorithm or report the run time of their algorithm on these two benchmarks. Therefore, we were not able to compare with these approaches. The details of the setup are as follows.

\revise{As stated earlier, the lengths of the PWL paths, i.e., $K_1, K_2, \cdots, K_N$, are constants. In the experiments, we set $K = K_1 = K_2 = \cdots = K_N$. Obviously, $K$ should be large enough, otherwise, the problem is not feasible. Thus, we start from $K = \underline{K}$, where $\underline{K}$ is an initial guess by the user according to the task. If the problem is infeasible, we increment $K$ by $1$ until the problem becomes feasible.}

In~\cite{raman2014model}, the authors proposed an MPC-based method of planning paths from STL specifications. It also models the planning as an MILP problem. The decision variables are just the states of the system at $\Delta t, 2 \Delta t, \cdots, \lceil \frac{T}{\Delta t} \rceil \Delta t$, where the time step $\Delta t > 0$ is a small constant specified by the user. The dynamics of the robots are encoded as constraint of the MILP problem. Furthermore, the authors of~\cite{raman2014model} proposed a group of rules with which an STL formula can be converted into linear constraints. In the experiments, we set $\Delta t = 0.1$. To make the comparison fair, we use a very simple dynamics $\dot{x} = u$ for MPC, i.e., an integrator. As for the inter-agent collision avoidance requirement, we represent it as constraints that at each time step, the distance between any pair of agents must be greater than a threshold. Obviously, the performance of MPC highly relies on the time horizon $T$. However, we do not have an idea of how large $T$ should be for completing each benchmark. In order to determine a good time bound that is not too large but large enough for completing the task, we first run our algorithm with $T = 1000$ which is large enough for all the benchmarks in this section. Our algorithm returns a PWL path with the optimized travel time\footnote{The solution is not exactly optimal. The precision of the solution depends on one of Gurobi's arguments, ``\texttt{MIPGap}". We always uses the same ``\texttt{MIPGap}" for MPC and our method.}. Then the makespan of the planned PWL paths is used as the $T$ when running the MPC-based algorithm. \revise{Therefore, both the proposed method and the MPC-based method need a pre-process to determine $K$ or $N$. To make the comparison clear and fair, we did not include the time spent for this pre-process in TABLE~\ref{tab:results}.}

We also implement an abstraction-based method based on~\cite{brown2020optimal, ma2019searching}, which uses a MILP-based approach for optimal task assignment and ordering, and leverages the priority-based search to plan collision-free trajectories to achieve all the assigned tasks. It does not support general STL specifications but supports the tasks in \texttt{wall-1} and \texttt{wall-2}.

\subsection{Observations}
\begin{table}[tbp]
    \centering
\begin{tabular}{|c|c|c|c|}
    \hline
    Benchmark & Ours (s) & MPC (s)& ABS (s)\\
    \hline
    \texttt{stlcg-1} & \textbf{0.855} & 6.5 & N/A \\ \hline
    \texttt{stlcg-2} & \textbf{0.175} & 4.0 & N/A \\ \hline
    \texttt{doorpuzzle-1} & \textbf{49.5} & TO & N/A \\ \hline
    \texttt{doorpuzzle-2} & \textbf{175.6} & 2102.5 & N/A \\ \hline
    \texttt{rover-1} & \textbf{180.5} & TO & N/A \\ \hline
    \texttt{rover-2} & \textbf{101.6} & 2733.7 & N/A \\ \hline
    \texttt{wall-1} & \textbf{20.8} & 113.7 & 50.7 \\ \hline
    \texttt{wall-2} & 172.9 & 163.7 & \textbf{79.1} \\ \hline
\end{tabular}
\caption{Run time on benchmarks. MPC failed in some cases due to time out (TO). ABS does not support general STL scenarios and can only handle the last two benchmarks.}
\label{tab:results}
\end{table}

The results are summarized in TABLE~\ref{tab:results}. Planned paths can be found in Fig.~\ref{fig:benchmarks-1} and Fig.~\ref{fig:benchmarks-2}. Some observations are in order. Firstly, the proposed method can correctly solve planning problems with complex STL specifications for multiple agents (up to $4$) while other methods in comparison failed in some cases. Secondly, the proposed method outperforms other methods in almost all cases in terms of run time. Thirdly, as shown in Fig.~\ref{fig:benchmarks-1} and Fig.~\ref{fig:benchmarks-2}, because the tracking error is taken into account when planning the PWL paths, the actual trajectories of the robots satisfy the STL specification although they deviate from the reference PWL paths. Also, results show that our algorithm is able to correctly figure out the logical ordering of events with temporal constraints, then automatically assign tasks to each agent and do essential arrangement to avoid inter-agent collision.
It is also worth mentioning that the tool (\texttt{stlcg}) proposed in~\cite{leung2020back} also uses a fixed time step and uses gradient decent to minimize the violation of the STL specification. It takes minutes to find paths for its two benchmark scenarios, \texttt{stlcg-1} and \texttt{stlcg-2}, while our method takes less than one second.

Furthermore, we evaluated the proposed approach on selected benchmarks, including \texttt{doorpuzzle-1}, \texttt{doorpuzzle-2}, \texttt{wall-1}, and \texttt{wall-2}, with real-world robots on the Robotarium~\cite{pickem2017robotarium} platform. \revise{For the real robots, we use the official tracking controller provided by the Robotarium team, and the tracking error is estimated from simulations using the official simulator. Experiments show that with the proposed approach and the tracking controller, the robots can safely complete the tasks.} Videos can be found in the supplementary material.

\section{Conclusion}
We introduced a novel method to synthesize long-horizon motions of multi-agent robotic systems for STL specifications.
Our method can effectively encode complex specifications and support long-time horizon synthesis due to the combinatorial use of PWL reference paths and guaranteed tracking controller.
We plan to further reduce the complexity of the encoding rules and support planning for larger-scale problems.

\bibliographystyle{IEEEtran}
\bibliography{IEEEabrv,references}

\ifx\extended\undefined
\else
\newpage 
\onecolumn
\appendix
\subsection{Soundness of the STL encoding}
The following lemmas will be used.
\begin{lemma}
\label{lemma:outside}
Given a time-stamped line segment $S$ whose endpoints are $(t_1, p_1)$ and $(t_2, p_2)$, a polytope $\poly{(H,b)}$, and $\epsilon > 0$, if
\[
\bigvee_{j=1}^{\polydim{H}}(H^{(j)} \cdot p_1 - b^{(j)} - \epsilon \|H^{(j)}\|_2 \geq 0 \wedge H^{(j)} \cdot p_{2} - b^{(j)} - \epsilon \|H^{(j)}\|_2 \geq 0)
\]
is true, then for all $t \in [t_1, t_2]$, we have that $B_{\epsilon}(S(t)) \cap \poly{(H,b)} = \varnothing$.
\end{lemma}
\begin{proof}
There exists $j \in \{1,\cdots,\polydim{H}\}$ such that $H^{(j)} \cdot p_1 - b^{(j)} - \epsilon \|H^{(j)}\|_2 \geq 0$ and $H^{(j)} \cdot p_{2} - b^{(j)} - \epsilon \|H^{(j)}\|_2 \geq 0$, which means both $p_1$ and $p_2$ are in the half-space defined by $\mathtt{HS} := \{q \in \mathcal{W}\ | \ H^{(j)} \cdot q - b^{(j)} - \epsilon \|H^{(j)}\|_2 \geq 0\}$. Thus, for all $t \in [t_1, t_2]$, $S(t) \in \mathtt{HS}$. Also, every point in $\mathtt{HS}$ is outside $\poly{(H,b)}$ with a margin greater than $\epsilon$. Therefore, for all $t \in [t_1, t_2]$, we have that $B_{\epsilon}(S(t)) \cap \poly{(H,b)} = \varnothing$.
\end{proof}

\begin{lemma}
\label{lemma:inside}
Given a time-stamped line segment $S$ whose endpoints are $(t_1, p_1)$ and $(t_2, p_2)$, a polytope $\poly{(H,b)}$, and $\epsilon > 0$, if
\[
\bigwedge_{j=1}^{\polydim{H}}(b^{(j)} - H^{(j)} \cdot p_1 - \epsilon \|H^{(j)}\|_2 \geq 0 \wedge b^{(j)} - H^{(j)} \cdot p_2 - \epsilon \|H^{(j)}\|_2 \geq 0)
\]
is true, then for all $t \in [t_1, t_2]$, we have that $B_{\epsilon}(S(t)) \subseteq \poly{(H,b)}$.
\end{lemma}
\begin{proof}
Instead of considering bloating the line segment, we consider shrinking the polytope. The new polytope $\mathtt{P} := \{q \in \mathcal{W}\ | \ \bigwedge_{j=1}^{\polydim{H}}b^{(j)} - H^{(j)} \cdot p - \epsilon \|H^{(j)}\|_2 \geq 0\}$ is obtained by shrinking the original one by $\epsilon$, i.e., moving each face of the polytope along the normal vector pointing inside by $\epsilon$. Since both $p_1$ and $p_2$ are inside $\mathtt{P}$, we have that the whole segment are contained in $\mathtt{P}$. Thus, for all $t \in [t_1, t_2]$, we have that $B_{\epsilon}(S(t)) \subseteq \poly{(H,b)}$.
\end{proof}

\begin{theorem}
Consider an STL formula $\varphi$, a PWL path $S$ whose timed waypoints are $\{t_i, p_i)\}_{i=0}^K$, and the LCFs $z_i^\varphi, i=0,\cdots,K-1$, constructed using the proposed encoding algorithm. For {\bf any} position trajectory $p$ satisfying $\forall t \geq 0$, $\|p(t) - S(t)\| \leq \epsilon$, if $z_i^\varphi$ is true, then $\forall t \in [t_i, t_{i+1}]$, $(p, t) \sat \varphi$.
\label{thm:STL-sound}
\end{theorem}
\begin{proof}
We proceed by induction. First, by Lemma~\ref{lemma:outside} and Lemma~\ref{lemma:inside}, the statement of the theorem is correct for any atomic predicates or their negations, i.e., $\varphi =$ $\pi^\mu$ or $\pi^{\neg \mu}$, which is the base case of the induction. Then, the inductive step is as follows. We will prove that, for a formula $\varphi$, {\em if the statement of the theorem is correct for its sub-formulas (this is the induction hypothesis), then it is also correct for $\varphi$}. For example, considering the encoding of the ``always" operation $z_i^{\always_{[a,b]}\varphi}$, we will show that if $\forall i, ~ z_i^\varphi \implies \forall t \in [t_i, t_{i+1}]$, $(p, t) \sat \varphi$, then $\forall i, ~ z_i^{\always_{[a,b]}\varphi} \implies \forall t \in [t_i, t_{i+1}]$, $(p, t) \sat \always_{[a,b]}\varphi$.
\newline
\newline
\noindent For conjunctions and disjunctions:
\paragraph{Conjunction $\varphi_1 \wedge \varphi_2$}
If $z_i^{\varphi_1 \wedge \varphi_2}$ is true, then both $z_i^{\varphi_1}$ and $z_i^{\varphi_2}$ are true, which implies that for all $t \in [t_i, t_{i+1}]$, $(p, t) \sat \varphi_1$ and $(p, t) \sat \varphi_2$. This further implies $(p, t) \sat \varphi_1 \wedge \varphi_2$, $\forall t \in [t_i, t_{i+1}]$.

\paragraph{Disjunction $\varphi_1 \vee \varphi_2$}
If $z_i^{\varphi_1 \vee \varphi_2}$ is true, then at least one of $z_i^{\varphi_1}$ and $z_i^{\varphi_2}$ is true. Without loss of generality, assume that $z_i^{\varphi_1}$ is true. The, for all $t \in [t_i, t_{i+1}]$, $(p, t) \sat \varphi_1$, which further implies $(p, t) \sat \varphi_1 \vee \varphi_2$, $\forall t \in [t_i, t_{i+1}]$. 
\newline
\newline
\noindent For temporal operators:
\paragraph{Always $\always_{[a,b]}\varphi$}
Assume that $z_i^{\always_{[a,b]}\varphi}$ is true. For any $t \in [t_i, t_{i+1}]$ and any $t^\prime \in [t+a, t+b]$, $t^\prime$ must fall into a time interval $[t_j, t_{j+1}]$. Moreover, we also have that $t^\prime \in [t_i+a, t_{i+1}+b]$, and thus $[t_j, t_{j+1}]$ and $[t_i+a, t_{i+1}+b]$ must have non-empty intersection. According to the encoding, this implies that $z_j^\varphi$ is true. Finally, since $t^\prime \in [t_j, t_{j+1}]$, by induction hypothesis, we have that $(p, t^\prime) \sat \varphi$. To summarize, we have proved that $\forall t \in [t_i,t_{i+1}]$, $\forall t^{\prime} \in \left[t+a, t+b\right]$, we have that $\left(p, t^{\prime}\right) \vDash \varphi$.

\paragraph{Eventually $\eventually_{[a,b]}\varphi$}
If $z_i^{\eventually_{[a,b]}\varphi}$ is true, then according to the encoding, there exists $j$ such that $[t_{j}, t_{j+1}]$ and $[t_{i+1}+a, t_i+b]$ have non-empty   intersection and $z_j^\varphi$ is true. Then, we can pick any $t^\prime \in [t_{j}, t_{j+1}] \cap [t_{i+1}+a, t_i+b]$. By induction hypothesis, we have that $(p, t^\prime) \sat \varphi$. Moreover, for any $t \in [t_i, t_{i+1}]$, we know that $t+a \leq t^\prime \leq t+b$. To summarize, we have proved that $\forall t \in [t_i, t_{i+1}]$, $\exists t^\prime \in [t+a, t+b]$ such that $(p, t^\prime) \sat \varphi$.

\paragraph{Until $\varphi_1 \until_{[a,b]} \varphi_2$}
Similar to the proof for ``eventually", if $z_i^{\varphi_1 \until_{[a,b]} \varphi_2}$ is true, then for any $t \in [t_i, t_{i+1}]$, there exists $t^\prime \in [t+a, t+b]$ such that $(p, t^\prime) \sat \varphi_2$. Moreover, for any $t^{\prime\prime} \in [t, t^\prime]$, assume that $t^{\prime\prime} \in [t_l, t_{l+1}]$ for some $l$. We also have that $t^{\prime\prime} \in [t_i, t_{i+1}+b]$, and thus $[t_l, t_{l+1}]$ and $[t_i, t_{i+1}+b]$ must have intersection. According to the encoding of $z_i^{\varphi_1 \until_{[a,b]} \varphi_2}$, this implies that $z_l^{\varphi_1}$ is true. Finally, since $t^{\prime\prime} \in [t_l, t_{l+1}]$, we have that $(p, t^{\prime\prime}) \sat \varphi_1$ by induction hypothesis. To summarize, we have prove that $\forall t \in [t_i,t_{i+1}]$, $\exists t^{\prime} \in \left[t+a, t+b\right]$ such that $\left(p, t^{\prime}\right) \vDash \varphi_2$ and $\forall t^{\prime \prime} \in \left[t, t^{\prime}\right]$ we have that $\left(p, t^{\prime \prime}\right) \vDash \varphi_1$.

\paragraph{Release $\varphi_1 \release_{[a,b]} \varphi_2$}
Assume that $z_i^{\varphi_1 \release_{[a,b]} \varphi_2}$ is true. For any $t \in [t_i, t_{i+1}]$ and any $t^\prime \in [t+a, t+b]$, $t^\prime$ must fall into a time interval $[t_j, t_{j+1}]$. Moreover, we also have that $t^\prime \in [t_i+a, t_{i+1}+b]$, and thus $[t_j, t_{j+1}]$ and $[t_i+a, t_{i+1}+b]$ must have intersection. According to the encoding, this implies that at least one of the following holds: (1). $z_j^{\varphi_2}$ is true, and thus $(p, t^\prime) \sat \varphi_2$ by induction hypothesis; or (2). there exists $l < j$ such that $[t_{l}, t_{l+1}]$ and $[t_{i+1}, t_{i+1}+b]$ have intersection and $z_l^{\varphi_1}$ is true. In case (2), pick any $t^{\prime\prime} \in [t_{l}, t_{l+1}] \cap [t_{i+1}, t_{i+1}+b]$, we have that $t < t^{\prime\prime} \leq t^{\prime}$ and $(p, t^{\prime\prime}) \sat \varphi_1$. To summarize, we have prove that $\forall t \in [t_i,t_{i+1}]$, $\forall t^{\prime} \in \left[t+a, t+b\right]$, we have that $\left(p, t^{\prime}\right) \vDash \varphi_2$ or $\exists t^{\prime \prime} \in \left[t, t^{\prime}\right]$ such that $\left(p, t^{\prime \prime}\right) \vDash \varphi_1$.
\newline
\newline
So far we have proved for each operator, the inductive step is correct. For any STL formula $\varphi$, starting from its atomic predicates and using the above results, we can prove that the statement of the theorem is correct for $\varphi$.
\end{proof}

By assumption, we know that the actual position trajectory $p$ of the robot deviates from the PWL path $S$ up to the tracking error $\epsilon$. Together with Theorem~\ref{thm:STL-sound}, we have that $z_0^\varphi$ is true implies that $p \sat \varphi$.

\subsection{Soundness of the inter-agent collision avoidance encoding}
Formally, we have the following theorem.
\begin{theorem}
For the two time-stamped line segments $\mathtt{SEG}_1$ and $\mathtt{SEG}_2$ introduced in Section.~\ref{sec:mult}. If $\safe{\mathtt{SEG}_1, \mathtt{SEG}_1, 2\epsilon + s_1 + s_2}$ is true, then $\forall t \in [t_{11},t_{12}] \cap [t_{21}, t_{22}]$, $B_{s_1+\epsilon}(\mathtt{SEG}_1(t)) \cap B_{s_2+\epsilon}(\mathtt{SEG}_2(t)) = \varnothing$.
\label{thm:mult-sound}
\end{theorem}
\begin{proof}
If $\safe{\mathtt{SEG}_1, \mathtt{SEG}_1, 2\epsilon + s_1 + s_2}$ is true, then we have that (1). $[t_{11}, t_{12}] \cap [t_{21}, t_{22}] = \varnothing$, or (2). $\left\|\frac{p_{11}+p_{12}}{2} - \frac{p_{21}+p_{22}}{2}\right\|_1  \geq \left\|\frac{p_{11}-p_{12}}{2}\right\|_1 + (s_1 + \epsilon) \sqrt{\delta} + \left\|\frac{p_{21}-p_{22}}{2}\right\|_1 + (s_2 + \epsilon) \sqrt{\delta}$.

If (1) holds, then the statement of the theorem is trivially true. If (2) holds, then $\forall t \in [t_{11},t_{12}] \cap [t_{21}, t_{22}]$, since we have the following,
\begin{align*}
& \|\mathtt{SEG}_1(t) - \mathtt{SEG}_2(t)\|_1 + \left\|\frac{p_{11}-p_{12}}{2}\right\|_1 + \left\|\frac{p_{21}-p_{22}}{2}\right\|_1 \\
\geq & \|\mathtt{SEG}_1(t) - \mathtt{SEG}_2(t)\|_1 + \left\|\frac{p_{11}+p_{12}}{2} - \mathtt{SEG}_1(t)\right\|_1 + \left\|\mathtt{SEG}_2(t) - \frac{p_{21}+p_{22}}{2}\right\|_1 \\
\geq & \left\|\frac{p_{11}+p_{12}}{2} - \frac{p_{21}+p_{22}}{2}\right\|_1,
\end{align*}
where the first inequality follows from the fact that $\mathtt{SEG}_1(t)$ and $\mathtt{SEG}_2(t)$ are points on the corresponding line segments, and the second ineuqlity follows from the triangle inequality, we have that
\begin{align*}
& \sqrt{\delta}\|\mathtt{SEG}_1(t) - \mathtt{SEG}_2(t)\|_2 \geq \|\mathtt{SEG}_1(t) - \mathtt{SEG}_2(t)\|_1\\
\geq & \left\|\frac{p_{11}+p_{12}}{2} - \frac{p_{21}+p_{22}}{2}\right\|_1 - \left(\left\|\frac{p_{11}-p_{12}}{2}\right\|_1 + \left\|\frac{p_{21}-p_{22}}{2}\right\|_1\right) \geq (s_1 + \epsilon + s_2 + \epsilon)\sqrt{\delta}.
\end{align*}
Thus, $\|\mathtt{SEG}_1(t) - \mathtt{SEG}_2(t)\|_2 \geq (s_1 + \epsilon) + (s_2 + \epsilon)$, which implies $B_{s_1 + \epsilon}(\mathtt{SEG}_1(t)) \cap B_{s_1 + \epsilon}(\mathtt{SEG}_2(t)) = \varnothing$.
\end{proof}

Again, by assumption, we know that for each agent $i$, its actual position trajectory deviates from its PWL reference path up to the tracking error $\epsilon$. Together with Theorem~\ref{thm:mult-sound}, we have that $z_{\texttt{inter}}$ is true implies that all the agents will not collide with any other in the presence of any disturbances.
\fi
\end{document}